%% file: vbmisspec.tex
\documentclass{article}

\PassOptionsToPackage{numbers}{natbib}
\usepackage[nonatbib,final]{neurips_2019}

\usepackage{setspace}
\usepackage{makecell}
\newcolumntype{M}[1]{>{\centering\arraybackslash}m{#1}}
\usepackage[T1]{fontenc}

\usepackage{tgtermes}
\usepackage{amsmath}
\usepackage{scalefnt,letltxmacro}
\LetLtxMacro{\oldtextsc}{\textsc}
\renewcommand{\textsc}[1]{\oldtextsc{\scalefont{1.10}#1}}
\usepackage[scaled=0.92]{PTSans}
\usepackage{inconsolata}
\usepackage{mathbbol}
\usepackage{amsthm}
\usepackage{amssymb}

\usepackage{setspace}

\usepackage[usenames,dvipsnames]{xcolor}
\definecolor{shadecolor}{gray}{0.9}

\usepackage[english]{babel}
\usepackage{afterpage}
\usepackage{framed}
\usepackage{nicefrac}

{\endMakeFramed}

\DeclareRobustCommand{\parhead}[1]{\textbf{#1}~}

\usepackage{lineno}

\usepackage{ragged2e}

\usepackage{marginnote}

\newcounter{parcount}

\usepackage{graphicx}
\usepackage[labelfont=bf]{caption}
\usepackage[format=hang]{subcaption}
\usepackage{wrapfig}

\usepackage{booktabs}
\usepackage{multirow}
\usepackage{longtable}

\usepackage{natbib}
\usepackage{bibunits}

\usepackage[algoruled]{algorithm2e}
\usepackage{listings}
\usepackage{fancyvrb}
\fvset{fontsize=\normalsize}

\SetKwComment{Comment}{$\triangleright$\ }{}

\usepackage[colorlinks,linktoc=all]{hyperref}
\usepackage[all]{hypcap}
\hypersetup{citecolor=Violet}
\hypersetup{linkcolor=black}
\hypersetup{urlcolor=MidnightBlue}

\usepackage[nameinlink]{cleveref}

\usepackage[acronym,smallcaps,nowarn]{glossaries}
\usepackage{listings}
\usepackage{lstbayes}
\lstdefinestyle{mystyle}{
    commentstyle=\color{OliveGreen},
    numberstyle=\tiny\color{black!60},
    stringstyle=\color{BrickRed},
    basicstyle=\ttfamily\scriptsize,
    breakatwhitespace=false,
    breaklines=true,
    captionpos=b,
    keepspaces=true,
    numbers=none,
    numbersep=5pt,
    showspaces=false,
    showstringspaces=false,
    showtabs=false,
    tabsize=2
}
\usepackage{enumitem}

\usepackage{centernot}

\DeclareMathOperator*{\argmax}{arg\,max}
\DeclareMathOperator*{\argmin}{arg\,min}

\crefname{lemma}{lemma}{lemmas}
\Crefname{lemma}{Lemma}{Lemmas}
\crefname{thm}{theorem}{theorems}
\Crefname{thm}{Theorem}{Theorems}
\crefname{prop}{proposition}{propositions}
\Crefname{prop}{Proposition}{Propositions}
\crefname{assumption}{assumption}{assumptions}
\Crefname{assumption}{Assumption}{Assumptions}

\usepackage[nameinlink]{cleveref}
\creflabelformat{equation}{#1#2#3}
\crefname{equation}{eq.}{eqs.}
\Crefname{equation}{Eq.}{Eqs.}
\Crefname{section}{\S}{\S}

\newtheorem{thm}{Theorem} % reset theorem numbering for each chapter
\newtheorem{lemma}[thm]{Lemma}
\newtheorem{assumption}{Assumption}
\newtheorem{corollary}[thm]{Corollary}

\renewcommand{\mid}{~\vert~}

\newcommand\dif{\mathop{}\!\mathrm{d}}

\newcommand{\cN}{\mathcal{N}}

\newcommand{\g}{\, | \,}
\newcommand{\s}{\, ; \,}

\newcommand{\E}[2]{\mathbb{E}_{#1}\left[#2\right]}

\usepackage{booktabs,arydshln}
\makeatletter
\def\adl@drawiv#1#2#3{%
        \hskip.5\tabcolsep
        \xleaders#3{#2.5\@tempdimb #1{1}#2.5\@tempdimb}%
                #2\z@ plus1fil minus1fil\relax
        \hskip.5\tabcolsep}
\newcommand{\cdashlinelr}[1]{%
  \noalign{\vskip\aboverulesep
           \global\let\@dashdrawstore\adl@draw
           \global\let\adl@draw\adl@drawiv}
  \cdashline{#1}
  \noalign{\global\let\adl@draw\@dashdrawstore
           \vskip\belowrulesep}}
\makeatother

\newenvironment{proofsk}{%
  \proof}{\endproof}

\newacronym{ELBO}{elbo}{evidence lower bound}
\newacronym{GMM}{gmm}{Gaussian mixture model}
\newacronym{KL}{kl}{Kullback-Leibler}
\newacronym{LDA}{lda}{latent Dirichlet allocation}
\newacronym{SVI}{svi}{stochastic variational inference}
\newacronym{VB}{vb}{variational Bayes}
\newacronym{TV}{tv}{total variation}
\newacronym{VBE}{vbe}{variational Bayes estimate}
\newacronym{VFE}{vfe}{variational frequentist estimate}
\newacronym{LAN}{lan}{local asymptotic normality}

\newacronym{MLE}{mle}{maximum likelihood estimate}
\newacronym{MAP}{map}{maximum-a-posterior}
\newacronym{MCMC}{mcmc}{Markov chain Monte Carlo}
\newacronym{EM}{em}{expectation maximization}
\newacronym{LBFGS}{l-bfgs}{limited-memory Broyden-Fletcher-Goldfarb-Shanno}
\newacronym{ADVI}{advi}{automatic differentiation variational inference}
\newacronym{NUTS}{nuts}{No-U-Turn sampler}
\newacronym{HMC}{hmc}{Hamiltonian Monte Carlo}

\newacronym{GLM}{glm}{generalized linear model}
\newacronym{GLMM}{glmm}{generalized linear mixed model}
\newacronym{LMM}{lmm}{linear mixed model}
\newacronym{SBM}{sbm}{stochastic block model}

\newacronym{IF}{if}{influence function}

\newacronym{PF}{pf}{Poisson factorization}

\newacronym[\glsshortpluralkey={rpm}]
{RPM}{rpm}{reweighted probabilistic model}

\newacronym{SNR}{snr}{signal-to-noise ratio}

\usepackage{times}
\usepackage{adjustbox} 
\usepackage{tcolorbox}
\usepackage{cuted}
\usepackage{enumitem}

\everypar=\expandafter{\the\everypar\loosness=-1 }
\title{Variational Bayes under Model Misspecification}

\author{
  Yixin Wang\\
  Columbia University\\
  \And
  David M.~Blei\\
  Columbia University
}

\begin{document}

\maketitle

\begin{bibunit}[alp]
\input{sec_abstract}

\input{sec_intro}
\input{sec_misspecify}
\input{sec_misspec_gen}
% \input{sec_app}
\input{sec_simulation}
\input{sec_conclusion}

\parhead{Acknowledgments. } We thank Victor Veitch and Jackson Loper
for helpful comments on this article. This work is supported by ONR
N00014-17-1-2131, ONR N00014-15-1-2209, NIH 1U01MH115727-01, NSF
CCF-1740833, DARPA SD2 FA8750-18-C-0130, IBM, 2Sigma, Amazon, NVIDIA,
and Simons Foundation.
\clearpage
{\small\putbib[BIB1]}
\end{bibunit}

\clearpage
\setcounter{page}{1}
\begin{bibunit}[alp]
\input{sec_supp}

\clearpage
\putbib[BIB1]
\end{bibunit}

\end{document}

%% file: sec_abstract.tex
% !TEX root = vbmisspec.tex
\begin{abstract}
\Gls{VB} is a scalable alternative to \gls{MCMC} for Bayesian
posterior inference. Though popular, \gls{VB} comes with few
theoretical guarantees, most of which focus on well-specified models.
However, models are rarely well-specified in practice. In this work,
we study \gls{VB} under model misspecification. We prove the \gls{VB}
posterior is asymptotically normal and centers at the value that
\emph{minimizes} the \gls{KL} divergence to the true data-generating
distribution. Moreover, the \gls{VB} posterior mean centers at the
same value and is also asymptotically normal. These results generalize
the variational Bernstein--von Mises theorem
\citep{wang2018frequentist} to misspecified models. As a consequence
of these results, we find that the model misspecification error
\emph{dominates} the variational approximation error in \gls{VB}
posterior predictive distributions. It explains the widely observed
phenomenon that \gls{VB} achieves comparable predictive accuracy with
\gls{MCMC} even though \gls{VB} uses an approximating family. As
illustrations, we study \gls{VB} under three forms of model
misspecification, ranging from model over-/under-dispersion to latent
dimensionality misspecification. We conduct two simulation studies
that demonstrate the theoretical results.
\end{abstract}

%% file: sec_intro.tex
% !TEX root = vbmisspec.tex
\section{Introduction}
\label{sec:introduction}
\glsreset{VB}

Bayesian modeling uses posterior inference to discover patterns in
data.  Begin by positing a probabilistic model that describes the
generative process; it is a joint distribution of latent variables and
the data. The goal is to infer the posterior, the conditional
distribution of the latent variables given the data. The inferred
posterior reveals hidden patterns of the data and helps form
predictions about new data.~\looseness=-1

For many models, however, the posterior is computationally
difficult---it involves a marginal probability that takes the form of
an integral. Unless that integral admits a closed-form expression (or
the latent variables are low-dimensional) it is intractable to
compute.

To circumvent this intractability, investigators rely on approximate
inference strategies such as \gls{VB}.  \gls{VB} approximates the
posterior by solving an optimization problem. First propose an
approximating family of distributions that contains all factorizable
densities; then find the member of this family that minimizes the
\gls{KL} divergence to the (computationally intractable) exact
posterior. Take this minimizer as a substitute for the posterior and
carry out downstream data analysis.~\looseness=-1

\gls{VB} scales to large datasets and works empirically in many
difficult models.  However, it comes with few theoretical guarantees,
most of which focus on well-specified models. For example,
\citet{wang2018frequentist} establish the consistency and asymptotic
normality of the \gls{VB} posterior, assuming the data is generated by
the probabilistic model. Under a similar assumption of a
well-specified model, \citet{zhang2017convergence} derive the
convergence rate of the \gls{VB} posterior in settings with
high-dimensional latent variables.~\looseness=-1

But as George Box famously quipped, ``all models are wrong.''
Probabilistic models are rarely well-specified in practice. Does
\gls{VB} still enjoy good theoretical properties under model
misspecification? What about the \gls{VB} posterior predictive
distributions? These are the questions we study in this paper.

\parhead{Main idea.} We study \gls{VB} under model misspecification.
Under suitable conditions, we show that (1) the \gls{VB} posterior is
asymptotically normal, centering at the value that \emph{minimizes}
the \gls{KL} divergence from the true distribution; (2) the \gls{VB}
posterior mean centers at the same value and is asymptotically normal;
(3) in the variational posterior predictive, the error due to model
misspecification dominates the error due to the variational
approximation.

Concretely, consider $n$ data points $x_{1:n}$ independently and
identically distributed with a true density
$\prod_{i=1}^np_0(x_i)$. Further consider a parametric probabilistic
model with a $d$-dimensional latent variable $\theta = \theta_{1:d}$;
its density belongs to the family
$\{\prod_{i=1}^np(x_i \g \theta):\theta \in
\mathbb{R}^d\}$.\footnote{A parametric probabilistic model means the
  dimensionality of the latent variables do not grow with the number
  of data points. We extend these results to more general
  probabilistic models in \Cref{sec:misspecify-gen}.}  When the model
is misspecified, it does not contain the true density,
$p_0(x)\notin \{p(x \g \theta):\theta\in\Theta\}$.

Placing a prior $p(\theta)$ on the latent variable $\theta$, we infer
its posterior $p(\theta\g x_{1:n})$ using \gls{VB}. Mean-field
\gls{VB} considers an approximating family $\mathcal{Q}$ that includes
all factorizable densities
\begin{align*}
\mathcal{Q} = \left\{q(\theta): q(\theta) = \textstyle \prod_{i=1}^dq_i(\theta_i)\right\}.
\end{align*}
It then finds the member that minimizes the \gls{KL} divergence to the
exact posterior $p(\theta\g x_{1:n})$,
\begin{align}
  \label{eq:vboptimize}
  q^*(\theta) = \argmin_{q\in\mathcal{Q}}\gls{KL}(q(\theta)|| p(\theta \g x_{1:n})).
\end{align}
The global minimizer $q^*(\theta)$ is called the \gls{VB} posterior.
(Here we focus on mean-field \gls{VB}. The results below apply to \gls{VB}
with more general approximating families as well.)

We first study the asymptotic properties of the \gls{VB} posterior and
its mean. Denote $\theta^*$ as the value of $\theta$ that
minimizes the
\gls{KL} divergence to the true distribution,
\begin{align}
\label{eq:paramoptimaltheta}
\theta^* = \argmin_{\theta}\gls{KL}(p_0(x)|| p(x \g \theta)).
\end{align}
Note this KL divergence is different from the variational objective
(\Cref{eq:vboptimize}); it is a property of the model class's
relationship to the true density.  We show that, under standard
conditions, the \gls{VB} posterior $q^*(\theta)$ converges in
distribution to a point mass at $\theta^*$. Moreover, the \gls{VB}
posterior of the rescaled and centered latent variable $\tilde{\theta}
= \sqrt{n}(\theta - \theta^*)$ is asymptotically normal.  Similar
asymptotics hold for the \gls{VB} posterior mean
$\hat{\theta}_{\mathrm{VB}} = \int \theta\cdot q^*(\theta)\dif
\theta$: it converges almost surely to $\theta^*$ and is
asymptotically normal.~\looseness=-1

Why does the \gls{VB} posterior converge to a point mass at
$\theta^*$? The reason rests on three observations. (1) The classical
Bernstein--von Mises theorem under model misspecification
\citep{kleijn2012bernstein} says that the exact posterior
$p(\theta\g x_{1:n})$ converges to a point mass at $\theta^*$. (2)
Because point masses are factorizable, this limiting exact posterior
belongs to the approximating family $\mathcal{Q}$: if
$\theta^* = (\theta^*_1, \theta^*_2, \theta^*_3)$, then
$\delta_{\theta^*}(\theta) = \delta_{\theta^*_1}(\theta_1)\cdot
\delta_{\theta^*_2}(\theta_2) \cdot \delta_{\theta^*_3}(\theta_3)$.
(3) \gls{VB} seeks the member in $\mathcal{Q}$ that is closest to the
exact posterior (which also belongs to $\mathcal{Q}$, in the
limit). Therefore, the \gls{VB} posterior also converges to a point
mass at $\theta^*$.  \Cref{fig:intution} illustrates this
intuition---as we see more data, the posterior gets closer to the
variational family. We make this argument rigorous in
\Cref{sec:parametric}.

\begin{figure}[t]
\centering

\includegraphics[width=0.73\textwidth]{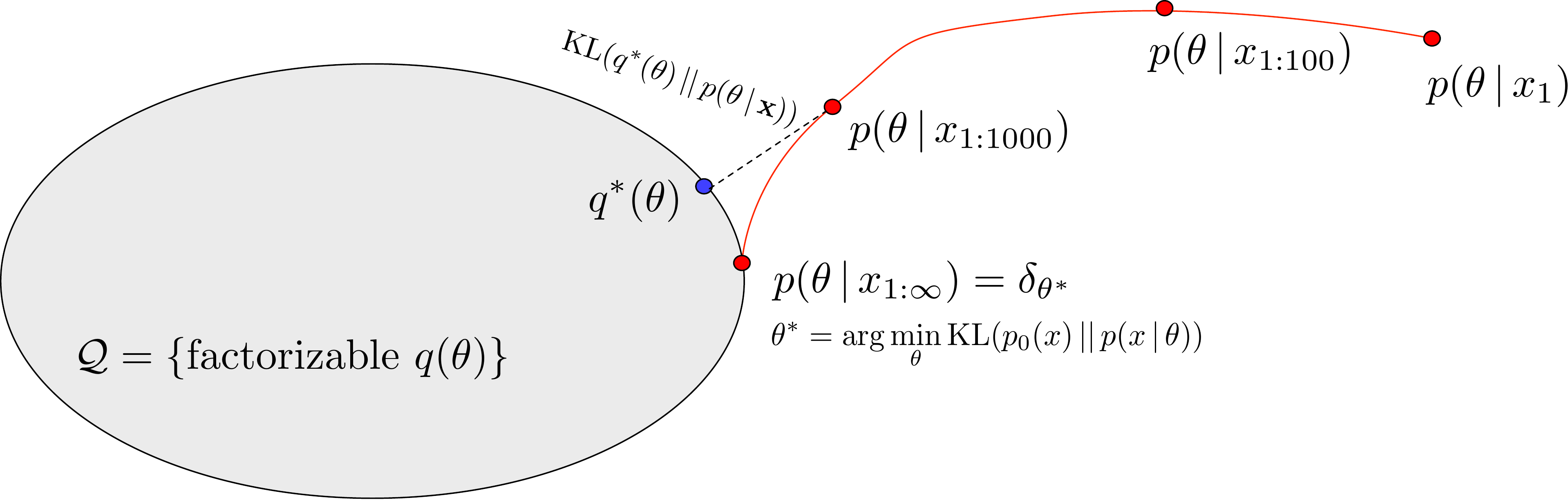}

\caption{Why does the \gls{VB} posterior converge to a point mass at
  $\theta^*$?  The intuition behind this figure is described in
  \Cref{sec:introduction}. In the figure, $q^*(x)$ is the optimal \gls{VB}
  posterior given $x_{1:1000}$.}
\label{fig:intution}
\end{figure}

The asymptotic characterization of the \gls{VB} posterior leads to an
interesting result about the \gls{VB} approximation of the posterior
predictive.  Consider two posterior predictive distributions under the
misspecified model.  The \gls{VB} predictive density is formed with
the \gls{VB} posterior,
\begin{align}
  \label{eq:vb-predictive}
  {p_{\gls{VB}}^{\mathrm{pred}}(x_\mathrm{new}\g x_{1:n}) = \int
  p(x_\mathrm{new} \g \theta)\cdot q^*(\theta)\dif \theta}.
\end{align}
The exact posterior predictive density is formed with the exact
posterior,
\begin{align}
  \label{eq:exact-predictive}
  p^{\mathrm{pred}}_{\mathrm{exact}}(x_\mathrm{new}\g x_{1:n}) = \int
  p(x_\mathrm{new}\g \theta)\cdot p(\theta\g x_{1:n})\dif \theta.
\end{align}

Now define the \textit{model misspecification error} to be the
\gls{TV} distance between the exact posterior predictive and the true
density $p_0(x)$.  (When the model is well-specified, it converges to
zero \citep{van2000asymptotic}.)  Further define the
\textit{variational approximation error} is the \gls{TV} distance
between the variational predictive and the exact predictive; it
measures the price of the approximation when using the \gls{VB}
posterior to form the predictive. Below we prove that the model
misspecification error dominates the variational approximation
error---the variational approximation error vanishes as the number of
data points increases. This result explains a widely observed
phenomenon: \gls{VB} achieves comparable predictive accuracy as
\gls{MCMC} even though \gls{VB} uses an approximating family
\citep{kucukelbir2016automatic,blei2017variational,blei2006variational,braun2010variational}.

The contributions of this work are to generalize the variational
Bernstein--von Mises theorem \citep{wang2018frequentist} to
misspecified models and to further study the \gls{VB} posterior
predictive distribution.  \Cref{sec:misspecify,sec:vbppd} details the
results around \gls{VB} in parametric probabilistic
models. \Cref{sec:misspecify-gen} generalizes the results to
probabilistic models where the dimensionality of latent variables can
grow with the number of data points.  \Cref{sec:appmain} illustrates
the results in three forms of model misspecification, including
underdispersion and misspecification of the latent dimensionality.
\Cref{sec:simulation} corroborates the theoretical findings with
simulation studies on \gls{GLMM} and \gls{LDA}.

\parhead{Related work.} This work draws on two themes around \gls{VB}
and model misspecification.

The first theme is a body of work on the theoretical guarantees of
\gls{VB}.  Many researchers have studied the properties of \gls{VB}
posteriors on particular Bayesian models, including linear models
\citep{you2014variational,ormerod2014variational}, exponential family
models \citep{wang2004convergence,wang2005inadequacy}, generalized
linear mixed models
\citep{hall2011theory,hall2011asymptotic,ormerod2012gaussian},
nonparametric regression \citep{faes2011variational}, mixture models
\citep{wang2006convergence,westling2015establishing}, stochastic block
models \citep{zhang2017theoretical,bickel2013asymptotic}, latent
Gaussian models \citep{sheth2017excess}, and latent Dirichlet
allocation \citep{ghorbani2018instability}. Most of these works assume
well-specified models, with a few exceptions including
\citep{sheth2017excess,ormerod2014variational,sheth2019pseudo}.

In another line of work, \citet{wang2018frequentist} establish the
consistency and asymptotic normality of \gls{VB} posteriors;
\citet{zhang2017convergence} derive their convergence rate; and
\citet{pati2017statistical} provide risk bounds of \gls{VB} point
estimates.  Further,
\citet{alquier2016properties,yang2017alpha}
study risk bounds for variational approximations of Gibbs posteriors
and fractional posteriors, 
\citet{jaiswal2019asymptotic} study $\alpha$-R\'{e}nyi-approximate
posteriors, and \citet{fan2018tap} and \citet{ghorbani2018instability}
study generalizations of \gls{VB} via TAP free energy. Again, most of
these works focus on well-specified models. In contrast, we focus on
\gls{VB} in general misspecified Bayesian models and characterize the
asymptotic properties of the \gls{VB} posterior and the \gls{VB}
posterior predictive.  Note, when the model is well-specified, our
results recover the variational Bernstein--von Mises theorem
of~\citep{wang2018frequentist}, but we further generalize their theory
and extend it to analyzing the posterior predictive distribution.
Finally,
\citet{alquier2017concentration,cherief2018consistency,alquier2016properties}
include results around variational approximations in misspecified
models. These results focus mostly on risk bounds while we focus on
the asymptotic distribution of variational posteriors.

The second theme is about characterizing posterior distributions under
model misspecification. Allowing for model misspecification,
\citet{kleijn2012bernstein} establishes consistency and asymptotic
normality of the exact posterior in parametric Bayesian models;
\citet{kleijn2006misspecification} studies exact posteriors in
infinite-dimensional Bayesian models. We leverage these results around
exact posteriors to characterize \gls{VB} posteriors and \gls{VB}
posterior predictive distributions under model misspecification. 

%%% Local Variables:
%%% mode: latex
%%% TeX-master: "vbmisspec"
%%% End:

%% file: sec_misspecify.tex
% !TEX root = vbmisspec.tex

\section{\glsreset{VB}\Gls{VB} under model misspecification}

\label{sec:parametric}

\Cref{sec:misspecify,sec:vbppd} examine the asymptotic properties of
\gls{VB} under model misspecification and for parametric
models. \Cref{sec:misspecify-gen} extends these results to more
general models, where the dimension of the latent variables grows with
the data.  \Cref{sec:appmain} illustrates the results with three types
of model misspecification.

\subsection{The \gls{VB} posterior and the \gls{VB} posterior mean}
\label{sec:misspecify}

We first study the \gls{VB} posterior $q^*(\theta)$ and its mean
$\hat{\theta}_{\gls{VB}}$. Assume iid data from a density $x_i \sim
p_0$ and a parametric model $p(x\g \theta)$, i.e., a model where the
dimension of the latent variables does not grow with the data.  We
show that the optimal variational distribution $q^*(\theta)$
(\Cref{eq:vboptimize}) is asymptotically normal and centers at
$\theta^*$ (\Cref{eq:paramoptimaltheta}), which minimizes the \gls{KL}
between the model $p_\theta$ and the true data generating distribution
$p_0$. The \gls{VB} posterior mean $\hat{\theta}_{\gls{VB}}$ also
converges to $\theta^*$ and is asymptotically normal.

Before stating these asymptotic results, we make a few assumptions
about the prior $p(\theta)$ and the probabilistic model $\{p(x\g
\theta):\theta\in\Theta\}$. These assumptions resemble the classical
assumptions in the Bernstein--von Mises theorems
\citep{van2000asymptotic,kleijn2012bernstein}.

\begin{assumption}[Prior mass]
\label{assumption:paramprior}
The prior density $p(\theta)$ is continuous and positive in a
neighborhood of $\theta^*$. There exists a constant $M_p>0$ such that
$|(\log p(\theta))''|\leq M_pe^{|\theta|^2}$.
\end{assumption}

\Cref{assumption:paramprior} roughly requires that the prior has some
mass around the optimal $\theta^*$. It is a necessary assumption: if
$\theta^*$ does not lie in the prior support then the posterior cannot
be centered there. \Cref{assumption:paramprior} also requires a tail
condition on $\log p(\theta)$: the second derivative of
$\log p(\theta)$ can not grow faster than $\exp(|\theta|^2)$. This is
a technical condition that many common priors satisfy.

\begin{assumption}[Consistent testability]
\label{assumption:paramtest} For every $\epsilon > 0$ there exists
    a sequence of tests $\phi_n$ such that
    \begin{align}
        \int \phi_n(x_{1:n})\prod_{i=1}^np_0(x_i)
        \dif x_{1:n}
        \rightarrow 0,
    \end{align}
    \begin{align}
        \sup_{\{\theta:||\theta -\theta^*||\geq \epsilon\}} 
        \int (1-\phi_n(x_{1:n}))\cdot\left[
        \prod_{i=1}^n\frac{p(x_i\g \theta)}{p(x_i\g \theta^*)}
        p_0(x_i)\right]
        \dif x_{1:n}
        \rightarrow 0.
    \end{align}
\end{assumption}

\Cref{assumption:paramtest} roughly requires $\theta^*$ to be the
unique optimum of the \gls{KL} divergence to the truth
(\Cref{eq:paramoptimaltheta}). In other words, $\theta^*$ is
identifiable from fitting the probabilistic model $p(x\g \theta)$ to
the data drawn from $p_0(x)$. To satisfy this condition, it suffices
to have the likelihood ratio $p(x\g \theta_1)/ p(x\g \theta_2)$ be a
continuous function of $x$ for all $\theta_1, \theta_2 \in \Theta$
(Theorem 3.2 of \citep{kleijn2012bernstein}).

\Cref{assumption:paramprior} and \Cref{assumption:paramtest} are
classical conditions required for the asymptotic normality of the
exact posterior \citet{kleijn2012bernstein}. They ensure that, for
every sequence $M_n\rightarrow \infty$,
\begin{align}
\int_\Theta \mathbb{1}(||\theta - \theta^*||>\delta_nM_n) \cdot p(\theta \g
x_{1:n}) \dif \theta \stackrel{P_0}{\rightarrow} 0,
\end{align}
for some constant sequence $\delta_n\rightarrow 0$. In other words,
the exact posterior $p(\theta\g x)$ occupies vanishing mass outside of
the $\delta_nM_n$-sized neighborhood of $\theta^*$. We note that the
sequence $\delta_n$ also plays a role in the following
\glsreset{LAN}\gls{LAN} assumption.

\begin{assumption}[\glsreset{LAN}\Gls{LAN}]
\label{assumption:paramlan}
For every compact set
    $K\subset
    \mathbb{R}^d$, there exist random vectors $\Delta_{n,\theta^*}$ bounded
    in probability and nonsingular matrices $V_{\theta^*}$ such that
    \begin{align}
      \sup_{h\in K}
      \left|\log \frac{p(x \g \theta^* + \delta_nh)}{p(x\g \theta^*)
      }- h^\top V_{\theta^*}\Delta_{n,\theta^*}
      + \frac{1}{2}h^\top V_{\theta^*}h\right|
      \stackrel{P_0}{\rightarrow} 0,
    \end{align}
    where $\delta_n$ is a $d\times d$ diagonal matrix that describes
    how fast each  dimension of the $\theta$ posterior converges to a
    point mass. We note that $\delta_n\rightarrow 0$ as $n
    \rightarrow \infty$.
\end{assumption}
This is a key assumption that characterizes the limiting normal
distribution of the \gls{VB} posterior. The quantities
$\Delta_{n,\theta^*}$ and $V_{\theta^*}$ determine the normal
distribution that the \gls{VB} posterior will converge to. The
constant $\delta_n$ determines the convergence rate of the \gls{VB}
posterior to a point mass.  Many parametric models with a
differentiable likelihood satisfy \gls{LAN}. We provide a more
technical description on how to verify \Cref{assumption:paramlan} in
\Cref{sec:paramlansatisfy}.

With these assumptions, we establish the asymptotic properties of the
\gls{VB} posterior and the \gls{VB} posterior mean.

\begin{thm}
  \label{thm:parammain} \emph{(Variational Bernstein--von Mises
    Theorem under model misspecification, parametric model version)} Under
  \Cref{assumption:paramprior,assumption:paramtest,assumption:paramlan},
\begin{enumerate}[leftmargin=*]
\item The \gls{VB} posterior converges to a point mass at $\theta^*$:
\begin{align}
q^*(\theta)\stackrel{d}{\rightarrow}\delta_{\theta^*}.
\end{align}

\item Denote $\tilde{\theta} = \delta_n^{-1} (\theta - \theta^*)$ as
the re-centered and re-scaled version of $\theta$. The \gls{VB}
posterior of $\tilde{\theta}$ is asymptotically normal:
\begin{align}
  \left\|q^*({\tilde{\theta}}) -
\cN({\tilde{\theta}} \s\Delta_{n,\theta^*},
    V'^{-1}_{\theta^*}))
   \right\|
  _{\mathrm{TV}}\stackrel{P_0}{\rightarrow} 0.
\end{align}
where $V'_{\theta^*}$ is diagonal and has the same diagonal terms as
the exact posterior precision matrix $V_{\theta^*}$.~\looseness=-1
\item The \gls{VB} posterior mean converges to $\theta^*$ almost
surely:
\begin{align}
\hat{\theta}_{\mathrm{VB}}\stackrel{a.s.}{\rightarrow} \theta^*.
\end{align}
\item The \gls{VB} posterior mean is also asymptotically normal:
\begin{align}
\delta_n^{-1}(\hat{\theta}_{\mathrm{VB}} -
\theta^*)\stackrel{d}{\rightarrow}\Delta_{\infty,\theta^*},
\end{align}
where $\Delta_{\infty,\theta^*}$ is the limiting distribution of the
random vectors $\Delta_{n,\theta^*}$: $\Delta_{n,\theta^*}
\stackrel{d}{\rightarrow}
\Delta_{\infty,\theta^*}$. Its distribution
is~${\Delta_{\infty,\theta^*} \sim
\cN\left(0, V_{\theta^*}^{-1}\E{P_0}{(\log
p(x\g\theta^*))'(\log
p(x\g\theta^*))'^\top}V_{\theta^*}^{-1}\right)}$.
\end{enumerate}
\end{thm}

\begin{proofsk}
The proof structure of \Cref{thm:parammain} mimics
\citet{wang2018frequentist} but extends it to allow for model
misspecification. In particular, we take care of the extra
technicality due to the difference between the true data-generating
measure $p_0(x)$ and the probabilistic model we fit
$\{p(x \g \theta):\theta\in\Theta\}.$

The proof proceeds in three steps:
\begin{enumerate}[leftmargin=*]
\item Characterize the asymptotic properties of the exact posterior:
\begin{align*}
p(\theta \g x)&\stackrel{d}{\rightarrow}\delta_{\theta^*},\\
\left\|p(\tilde{\theta}\g x) - \cN(\Delta_{n,\theta^*}, V^{-1}_{\theta^*})\right \|_{\mathrm{TV}}&\stackrel{P_0}{\rightarrow}0.
\end{align*}
This convergence is due to
\Cref{assumption:paramprior,assumption:paramtest}, and the classical
Bernstein--von Mises theorem under model misspecification
\citep{kleijn2012bernstein}.
\item Characterize the \gls{KL} minimizer of the limiting exact
posterior in the variational approximating family $\mathcal{Q}$:
\begin{align*}
\argmin_{q\in\mathcal{Q}} \gls{KL}(q(\theta)\,||\, p(\theta \g x)) &\stackrel{d}{\rightarrow}\delta_{\theta^*},\nonumber\\
\left\|\argmin_{q\in\mathcal{Q}} \gls{KL}(q(\tilde{\theta})\,||\, p(\tilde{\theta} \g x)) - \cN(\tilde{\theta} \s\Delta_{n,\theta^*},
    V'^{-1}_{\theta^*})\right \|_{\mathrm{TV}}&\stackrel{P_0}{\rightarrow}0,
\end{align*}
where $V'$ is diagonal and shares the same diagonal terms as $V$. The
intuition of this step is due to the observation that the point mass
is factorizable: $\delta_{\theta^*}\in \mathcal{Q}$. We prove it via
bounding the mass outside a neighborhood of $\theta^*$ under the
\gls{KL} minimizer $q^*(\theta)$.
\item Show that the \gls{VB} posterior approaches the \gls{KL}
minimizer of the limiting exact posterior as the number of data points
increases:
\begin{align*}
    \left\|
    q^*(\theta)
-\argmin_{q\in\mathcal{Q}^d} 
    \gls{KL}(q(\cdot)|| \delta_{\theta^*})
    \right\|
    _{\textrm{TV}}\stackrel{P_{0}}{\rightarrow} 0.\\
    \left\|
    q^*(\tilde{\theta})
-\argmin_{q\in\mathcal{Q}^d} 
    \gls{KL}(q(\cdot)|| \cN(\cdot \s\Delta_{n,\theta^*}, V^{-1}_{\theta^*}))
    \right\|
    _{\textrm{TV}}\stackrel{P_{0}}{\rightarrow} 0.
\end{align*}
The intuition of this step is that if two distributions are close,
then their \gls{KL} minimizer should also be close. In addition, the
\gls{VB} posterior is precisely the \gls{KL} minimizer to the exact
posterior: $q^*(\theta) =
\argmin_{q\in\mathcal{Q}^d}
    \gls{KL}(q(\theta)||p(\theta\mid x))$. We leverage
$\Gamma$-convergence to prove this claim.
\end{enumerate}

These three steps establish the asymptotic properties of the \gls{VB}
posterior under model misspecification (\Cref{thm:parammain}.1 and
\Cref{thm:parammain}.2): the \gls{VB} posterior converges to
$\delta_{\theta^*}$ and is asymptotically normal.

To establish the asymptotic properties of the \gls{VB} posterior mean
(\Cref{thm:parammain}.3 and \Cref{thm:parammain}.4), we follow the
classical argument in Theorem 2.3 of \citet{kleijn2012bernstein},
which leverages that the posterior mean is the Bayes estimator under
squared loss. The full proof is in
\Cref{sec:mainthmproof}.~\looseness=-1
\end{proofsk}

\Cref{thm:parammain} establishes the asymptotic properties of the
\gls{VB} posterior under model misspecification: it is asymptotically
normal and converges to a point mass at $\theta^*$, which minimizes
the \gls{KL} divergence to the true data-generating distribution. It
also shows that the \gls{VB} posterior mean shares similar convergence
and asymptotic normality.

\Cref{thm:parammain} states that, in the infinite data limit, the
\gls{VB} posterior and the exact posterior converge to the same point
mass.  The reason for this coincidence is (1) the limiting exact
posterior is a point mass and (2) point masses are factorizable and
hence belong to the variational approximating family $\mathcal{Q}$. In
other words, the variational approximation has a negligible effect on
the limiting posterior.

\Cref{thm:parammain} also shows that the \gls{VB} posterior has a
different covariance matrix from the exact posterior. The \gls{VB}
posterior has a diagonal covariance matrix but the covariance of the
exact posterior is not necessarily diagonal. However, the inverse of
the two covariance matrices match in their diagonal terms. This fact
implies that the entropy of the limiting \gls{VB} posterior is always
smaller than or equal to that of the limiting exact posterior (Lemma 8
of \citet{wang2018frequentist}), which echoes the fact that the
\gls{VB} posterior is under-dispersed relative to the exact posterior.

We remark that the under-dispersion of the \gls{VB} posterior does not
necessarily imply under-coverage of the \gls{VB} credible
intervals. The reason is that, under model misspecification, even the
credible intervals of the exact posterior cannot guarantee coverage
\citep{kleijn2012bernstein}.  Depending on how the model is
misspecified, the credible intervals derived from the exact posterior
can be arbitrarily under-covering or over-covering. Put differently,
under model misspecification, neither the \gls{VB} posterior nor the
exact posterior are reliable for uncertainty quantification.

Consider a well-specified model, where $p_0(x) = p(x\g {\theta_0})$
for some $\theta_0\in\Theta$ and $\theta^* = \theta_0$.  In this case,
\Cref{thm:parammain} recovers the variational Bernstein--von Mises
theorem \citep{wang2018frequentist}. That said,
\Cref{assumption:paramtest,assumption:paramlan} are
stronger than their counterparts for well-specified models; the reason
is that $P_0$ is usually less well-behaved than $P_{\theta_0}$.
\Cref{assumption:paramtest,assumption:paramlan} more
closely align with those required in characterizing the exact
posteriors under misspecification (Theorem 2.1 of
\citep{kleijn2012bernstein}).

\subsection{The \gls{VB} posterior predictive distribution}
\label{sec:vbppd}

We now study the posterior predictive induced by the \gls{VB}
posterior. As a consequence of \Cref{thm:parammain}, the error due to
model misspecification dominates the error due to the variational
approximation.

Recall that $p^{\mathrm{pred}}_{\gls{VB}}(x_{\mathrm{new}}\g x_{1:n})$
is the
\gls{VB} posterior predictive (\Cref{eq:vb-predictive}),
$p^{\mathrm{pred}}_{\mathrm{true}}(x_{\mathrm{new}}\g x_{1:n})$ is the
exact posterior predictive (\Cref{eq:exact-predictive}), $p_0(\cdot)$
is the true data generating density, and the \gls{TV} distance between
two densities $q_1$ and $q_2$ is
$\|q_1(x)-q_2(x)\|_{\mathrm{TV}}\triangleq \frac{1}{2}\int
|q_1(x)-q_2(x)|\dif x.$
\begin{thm}
  \label{corollary:paramposteriorpred}
  (The \gls{VB} posterior predictive distribution) If the
  probabilistic model is misspecified, i.e.
  $\left\|p_0(x) - p(x\g {\theta^*})\right\|_{\mathrm{TV}} > 0$, then
  the model approximation error dominates the variational
  approximation error:
  \begin{align}
    \label{eq:paramposteriorpred}
    \frac{\left\|p^{\mathrm{pred}}_{\gls{VB}}(x_\mathrm{new}\g x_{1:n}) - p^{\mathrm{pred}}_{\mathrm{exact}}(x_\mathrm{new}\g x_{1:n})\right\|_{\mathrm{TV}}}{\left\|p_0(x_\mathrm{new}) - p^{\mathrm{pred}}_{\mathrm{exact}}(x_\mathrm{new}\g x_{1:n})\right\|_{\mathrm{TV}}}\stackrel{P_0}{\rightarrow} 0,
  \end{align}
  under the regularity condition $\int\nabla^2_\theta p(x\g
  \theta^*)\dif x < \infty$ and
  \Cref{assumption:paramprior,assumption:paramlan,assumption:paramtest}.
\end{thm}

\begin{proofsk}
  \Cref{corollary:paramposteriorpred} is due to two observations: (1)
  in the infinite data limit, the \gls{VB} posterior predictive
  converges to the exact posterior predictive and (2) in the infinite
  data limit, the exact posterior predictive does \textit{not}
  converge to the true data-generating distribution because of model
  misspecification. Taken together, these two observations give
  \Cref{eq:paramposteriorpred}.

  The first observation comes from \Cref{thm:parammain}, which implies
  that both the \gls{VB} posterior and the exact posterior converge to
  the same point mass $\delta_{\theta^*}$ in the infinite data limit.
  Thus, they lead to similar posterior predictive distributions, which
  gives
\begin{align}
\label{eq:vberror}
  \left\|p^{\mathrm{pred}}_{\gls{VB}}(x_\mathrm{new}\g x_{1:n}) - p^{\mathrm{pred}}_{\mathrm{true}}(x_\mathrm{new}\g x_{1:n})\right\|_{\mathrm{TV}}\stackrel{P_0}{\rightarrow} 0.
\end{align}
Moreover, the model is assumed to be misspecified
$\left\|p_0(x) - p(x\g {\theta^*})\right\|_{\mathrm{TV}} > 0$, which
implies
\begin{align}
\label{eq:modelerror}
\left\|p_0(x_\mathrm{new}) -
p^{\mathrm{pred}}_{\mathrm{exact}}(x_\mathrm{new}\g
x_{1:n})\right\|_{\mathrm{TV}} \rightarrow c_0 > 0.
\end{align}
This fact shows that the model misspecification error does not vanish
in the infinite data limit. \Cref{eq:vberror} and \Cref{eq:modelerror}
imply \Cref{corollary:paramposteriorpred}. The full proof of
\Cref{corollary:paramposteriorpred} is in \Cref{sec:corollaryproof}.
\end{proofsk}

As the number of data points increases,
\Cref{corollary:paramposteriorpred} shows that the model
misspecification error dominates the variational approximation error.
The reason is that both the \gls{VB} posterior and the exact posterior
converge to the same point mass.  So, even though the \gls{VB}
posterior has an under-dispersed covariance matrix relative to the
exact posterior, both covariance matrices shrink to zero in the
infinite data limit; they converge to the same posterior predictive
distributions.

\Cref{corollary:paramposteriorpred} implies that when the model is
misspecified, \gls{VB} pays a negligible price in its posterior
predictive distribution. In other words, if the goal is prediction, we
should focus on finding the correct model rather than on correcting
the variational approximation. For the predictive ability of the
posterior, the problem of an incorrect model outweighs the problem of
an inexact inference.

\Cref{corollary:paramposteriorpred} also explains the phenomenon that
\gls{VB} predicts well despite being an approximate inference
method. As models are rarely correct in practice, the error due to
model misspecification often dominates the variational approximation
error.  Thus, on large datasets, \gls{VB} can achieve comparable
predictive performance, even when compared to more exact Bayesian
inference algorithms (like long-run \gls{MCMC}) that do not use
approximating families
\citep{kucukelbir2016automatic,blei2017variational,blei2006variational,braun2010variational}.

%%% Local Variables:
%%% mode: latex
%%% TeX-master: "vbmisspec"
%%% End:

%% file: sec_misspec_gen.tex
% !TEX root = vbmisspec.tex

\subsection{\glsreset{VB}\Gls{VB} in misspecified general
  probabilistic models}

\label{sec:misspecify-gen}

\Cref{sec:misspecify,sec:vbppd} characterize the \gls{VB} posterior,
the \gls{VB} posterior mean, and the \gls{VB} posterior predictive
distribution in misspecified parametric models. Here we extend these
results to a more general class of (misspecified) models with both
global latent variables $\theta = \theta_{1:d}$ and local latent
variables $z = z_{1:n}$.  This more general class allows the local
latent variables to grow with the size of the data.  The key idea is
to reduce this class to the simpler parametric models, via what we
call the ``variational model.''

Consider the following probabilistic model with both global and local
latent variables for $n$ data points $x=x_{1:n}$,
\begin{align}
\label{eq:model}
  p(\theta, x, z) = p(\theta) \textstyle \prod_{i=1}^n p(z_i\g \theta) p(x_i \g
z_i, \theta).
\end{align}
The goal is to infer $p(\theta\g x)$, the posterior of the global
latent variables.\footnote{This model has one local latent variable
  per data point.  But the results here extend to probabilistic models
  with $z=z_{1:d_n}$ and non i.i.d data $x=x_{1:n}$. We only require
  that $d$ stays fixed as $n$ grows but $d_n$ grows with $n$.}

\gls{VB} approximates the posterior of both global and local latent
variables $p(\theta, z\g x)$ by minimizing its \gls{KL} to the exact
posterior:
\begin{align}
\label{eq:vbobjective}
q^*(\theta)q^*(z)=q^*(\theta, z) = \argmin_{q\in\mathcal{Q}}\gls{KL}(q(\theta, z)||
p(\theta, z\g x)),
\end{align}
where
$\mathcal{Q} = \{q: q(\theta, z) = \prod^d_{i=1}q_{\theta_i}(\theta_i)
\prod^n_{j=1}q_{z_j}(z_j)\}$ is the approximating family that contains
all factorizable densities.  (The first equality is because
$q^*(\theta, z)$ belongs to the factorizable family $\mathcal{Q}$.)
The \gls{VB} posterior of the global latent variables $\theta_{1:d}$
is $q^*(\theta)$.

\gls{VB} for general probabilistic models operates in the same way as
for parametric models, except we must additionally approximate the
posterior of the local latent variables. Our strategy is to reduce the
general probabilistic model with \gls{VB} to a parametric model
(\Cref{sec:misspecify}). Consider the so-called variational
log-likelihood \citep{wang2018frequentist},
\begin{align}
\label{eq:var-loglike}
  \log p^{\mathrm{VB}}(x\g\theta) = \eta(\theta) + 
  \max_{q(z)\in\mathcal{Q}} \, \, \E{q(z)}
{\log p(x, z \g \theta)
  - \log q(z)},
\end{align}
where $\eta(\theta)$ is a log normalizer. Now construct the
\emph{variational model} with $p^{\mathrm{VB}}(x\g\theta)$ as the
likelihood and $\theta$ as the global latent variable. This model no
longer contains local latent variables; it is a parametric model.

Using the same prior $p(\theta)$, the variational model leads to a
posterior on the global latent variable
\begin{align}
  \pi^*(\theta \g x)
  &
    \triangleq
    \frac{p(\theta) p^{\mathrm{VB}}(x \g \theta)}
    {\int p(\theta) p^{\mathrm{VB}}(x \g \theta) \dif \theta}.
\end{align}
As shown in \citep{wang2018frequentist}, the \gls{VB} posterior, which
optimizes the variational objective, is close to $\pi^*(\theta \g x)$,
\begin{align}
  \label{eq:generalvbred}
  q^*(\theta) = \argmin_{q\in\mathcal{Q}}\gls{KL}(q(\theta)||\pi^*(\theta \g x)) + o_{P_0}(1).
\end{align}
Notice that \Cref{eq:generalvbred} resembles \Cref{eq:vboptimize}.
This observation leads to a reduction of \gls{VB} in general
probabilistic models to \gls{VB} in parametric probabilistic models
with an alternative likelihood $p^{\mathrm{VB}}(x\g\theta)$. This
perspective then allows us to extend
\Cref{thm:parammain,corollary:paramposteriorpred} in \Cref{sec:misspecify}
to general probabilistic models. 

More specifically, we define the optimal value $\theta^*$ as in
parametric models:
\begin{align}
\label{eq:optimaltheta}
\theta^* \stackrel{\Delta}{=} \argmax \gls{KL}(p_0(x)|| p^{\mathrm{VB}}(\theta \, ; \,
x)).
\end{align} 
This definition of $\theta^*$ coincides with the definition in
parametric models (\Cref{eq:paramoptimaltheta}) when the model is
indeed parametric.

Next we state the assumptions and results for the \gls{VB} posterior
and the \gls{VB} posterior mean for general probabilistic models.

\begin{assumption}[Consistent testability]
\label{assumption:gentest} For every $\epsilon > 0$ there exists
    a sequence of tests $\phi_n$ such that
    \begin{align}
        \int \phi_n(x)p_0(x)
        \dif x
        \rightarrow 0,
    \end{align}
    \begin{align}
        \sup_{\{\theta:||\theta -\theta^*||\geq \epsilon\}} 
        \int (1-\phi_n(x))
        \frac{p^{\mathrm{VB}}(x\g \theta)}{p^{\mathrm{VB}}(x\g \theta^*)}
        p_0(x)
        \dif x
        \rightarrow 0.
    \end{align}
\end{assumption}

\begin{assumption}[\glsreset{LAN}\Gls{LAN}]
\label{assumption:genlan}
For every compact set
    $K\subset
    \mathbb{R}^d$, there exist random vectors $\Delta_{n,\theta^*}$ bounded
    in probability and nonsingular matrices $V_{\theta^*}$ such that
    \begin{align}
        \sup_{h\in K}
        \left|\log \frac{p^{\mathrm{VB}}(x\g \theta^* + \delta_nh)}{p^{\mathrm{VB}}(x\g \theta^*) }- h^\top V_{\theta^*}\Delta_{n,\theta^*} 
        + \frac{1}{2}h^\top V_{\theta^*}h\right|
        \stackrel{P_0}{\rightarrow} 0,
    \end{align}
    where $\delta_n$ is a $d\times d$ diagonal matrix, where
    $\delta_n\rightarrow 0$ as $n \rightarrow \infty$. \\
\end{assumption}

\Cref{assumption:gentest,assumption:genlan} are analogous to
\Cref{assumption:paramtest,assumption:paramlan} except that we replace
the model $p(x\g \theta)$ with the variational model
$p^{\mathrm{VB}}(x\g \theta)$. In particular,
\Cref{assumption:genmodellan} is a \gls{LAN} assumption on
probabilistic models with local latent variables, i.e. nonparametric
models. While the \gls{LAN} assumption does not hold generally in
nonparametric models with infinite-dimensional
parameters~\citep{freedman1999wald}, there are a few nonparametric
models that have been shown to satisfy the \gls{LAN} assumption,
including generalized linear mixed models \citep{hall2011asymptotic},
stochastic block models \citep{bickel2013asymptotic}, and mixture
models \citep{westling2015establishing}. We illustrate how to verify
\Cref{assumption:gentest,assumption:genlan} for specific models in
\Cref{sec:app}. We refer the readers to Section 3.4 of
\citet{wang2018frequentist} for a detailed discussion on these
assumptions about the variational model.

Under
\Cref{assumption:paramprior,assumption:gentest,assumption:genlan},
\Cref{thm:parammain,corollary:paramposteriorpred} can be generalized
to general probabilistic models. The full details of these results
(\Cref{thm:genmain,corollary:genposteriorpred}) are in
\Cref{sec:extensions}.

\subsection{Applying the theory}
\label{sec:appmain}
To illustrate the theorems, we apply
\Cref{thm:parammain,corollary:paramposteriorpred,thm:genmain,corollary:genposteriorpred}
to three types of model misspecification: underdispersion in Bayesian
regression of count data, component misspecification in Bayesian
mixture models, and latent dimensionality misspecification with
Bayesian stochastic block models. For each model, we verify the
assumptions of the theorems and then characterize the limiting
distribution of their \gls{VB} posteriors. The details of these
results are in \Cref{sec:app}.

%%% Local Variables:
%%% mode: latex
%%% TeX-master: "vbmisspec"
%%% End:

%% file: sec_simulation.tex
% !TEX root = vbmisspec.tex

\section{Simulations}
\label{sec:simulation}

We illustrate the implications of
\Cref{thm:parammain,corollary:paramposteriorpred,thm:genmain,corollary:genposteriorpred}
with simulation studies.  We studied two models, Bayesian \gls{GLMM}
\citep{mccullagh1984generalized} and \gls{LDA} \citep{blei2003latent}.
To make the models misspecified, we generate datasets from an
``incorrect'' model and then perform approximate posterior inference.
We evaluate how close the approximate posterior is to the limiting
exact posterior $\delta_{\theta^*}$, and how well the approximate
posterior predictive captures the test sets.

To approximate the posterior, we compare \gls{VB} with \gls{HMC},
which draws samples from the exact posterior.  We find that both
achieve similar closeness to $\delta_{\theta^*}$ and comparable
predictive log likelihood on test sets. We use two automated inference
algorithms in Stan \citep{carpenter2015stan}: \gls{ADVI}
\citep{kucukelbir2016automatic} for \gls{VB} and \gls{NUTS}
\citep{hoffman2014nuts} for \gls{HMC}. We lay out the detailed
simulation setup in \Cref{sec:detailsim}.

\begin{figure}[t]
\hspace{-2pt}
\begin{subfigure}[b]{0.25\textwidth}
\centering
\includegraphics[width=\textwidth]{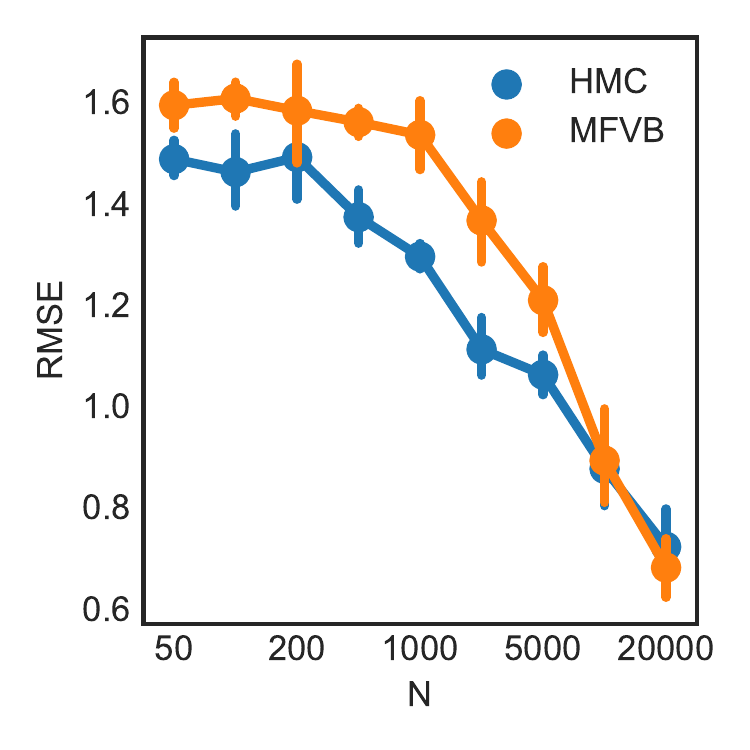}
\caption{\gls{GLMM}: RMSE to $\theta^*$\label{fig:mseplmm}}
\end{subfigure}
\begin{subfigure}[b]{0.25\textwidth}
\centering
\includegraphics[width=\textwidth]{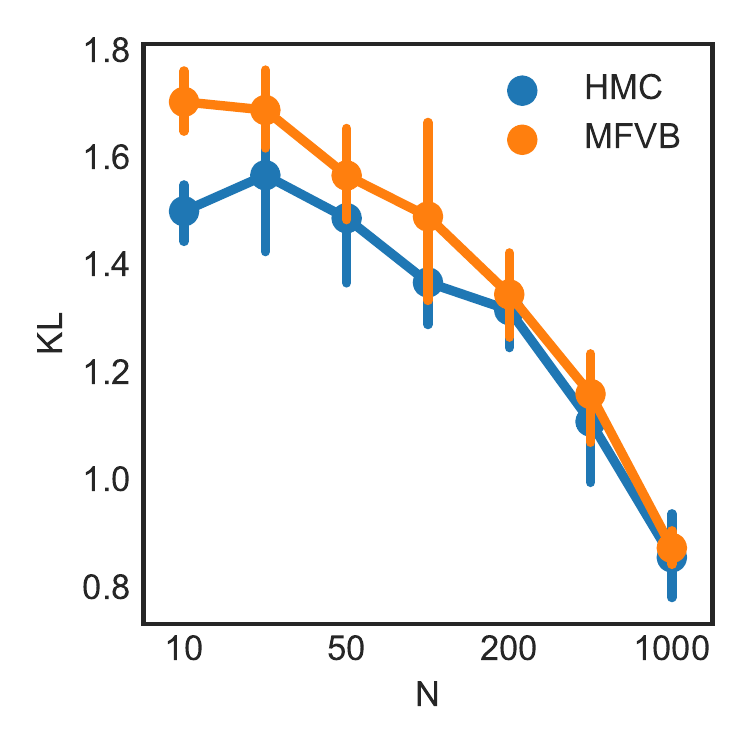}
\caption{\gls{LDA}: Mean \gls{KL} to $\theta^*$ \label{fig:kllda}}
\end{subfigure}%
\begin{subfigure}[b]{0.25\textwidth}
\centering
\includegraphics[width=\textwidth]{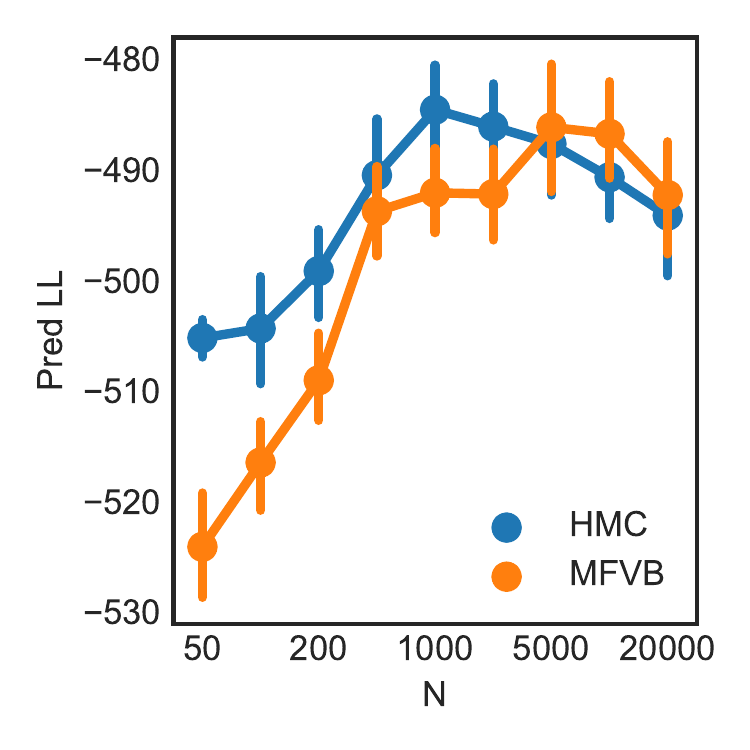}
\caption{\gls{GLMM}: Predictive LL \label{fig:pllplmm}}
\end{subfigure}%
\begin{subfigure}[b]{0.25\textwidth}
\centering
\includegraphics[width=\textwidth]{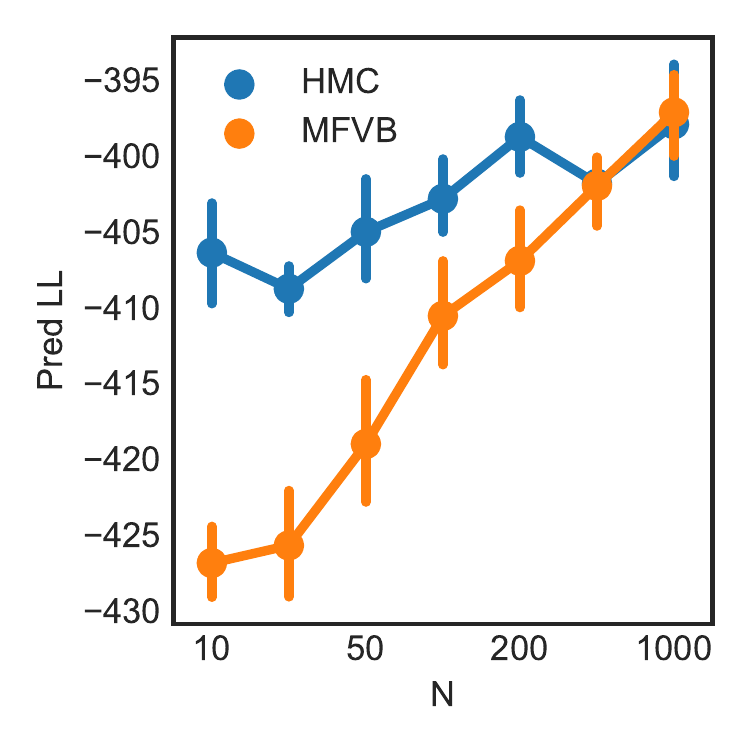}
\caption{\gls{LDA}: Predictive LL \label{fig:plllda}}
\end{subfigure}%
\caption{Dataset size versus closeness to the limiting exact posterior
$\delta_{\theta^*}$ and posterior predictive log likelihood on test
data (mean $\pm$ sd). \gls{VB} posteriors and \gls{MCMC} posteriors
achieve similar closeness to $\delta_{\theta^*}$ and comparable
predictive accuracy.
\label{fig:varygamma}}
\end{figure}

\parhead{Bayesian \gls{GLMM}.} We simulate data from a negative
binomial \gls{LMM}: each individual belongs to one of the ten groups;
each group has $N$ individuals; and the outcome is affected by a
random effect due to this group membership. Then we fit a Poisson
\gls{LMM} with the same group structure, which is misspecified with
respect to the simulated data. \Cref{fig:mseplmm} shows that the RMSE
to $\theta^*$ for the \gls{VB} and \gls{MCMC} posterior converges to
similar values as the number of individuals increases. This simulation
corroborates \Cref{thm:parammain,thm:genmain}: the limiting \gls{VB}
posterior coincide with the limiting exact
posterior. \Cref{fig:pllplmm} shows that \gls{VB} and \gls{MCMC}
achieve similar posterior predictive log likelihood as the dataset
size increases. It echoes
\Cref{corollary:paramposteriorpred,corollary:genposteriorpred}: when
performing prediction, the error due to the variational approximation
vanishes with infinite data.

\parhead{\glsreset{LDA}\Gls{LDA}.} We simulate $N$ documents from a
15-dimensional \gls{LDA} and fit a 10-dimensional \gls{LDA}; the
latent dimensionality of \gls{LDA} is misspecified. \Cref{fig:kllda}
shows the distance between the \gls{VB}/\gls{MCMC} posterior topics to
the limiting exact posterior topics, measured by \gls{KL} averaged
over topics. When the number of documents is at least 200, both
\gls{VB} and \gls{MCMC} are similarly close to the limiting exact
posterior. \Cref{fig:plllda} shows that, again once there are 200
documents, the \gls{VB} and \gls{MCMC} posteriors also achieve similar
predictive ability.  These results are consistent with
\Cref{thm:parammain,thm:genmain,corollary:paramposteriorpred,corollary:genposteriorpred}.

%%% Local Variables:
%%% mode: latex
%%% TeX-master: "vbmisspec"
%%% End:

%% file: sec_conclusion.tex
% !TEX root = vbmisspec.tex
\section{Discussion}
\label{sec:discussion}

In this work, we study \gls{VB} under model misspecification. We show
that the \gls{VB} posterior is asymptotically normal, centering at the
value that \emph{minimizes} the \gls{KL} divergence from the true
distribution. The \gls{VB} posterior mean also centers at the same
value and is asymptotically normal. These results generalize the
variational Bernstein--von Mises theorem \citet{wang2018frequentist}
to misspecified models. We further study the \gls{VB} posterior
predictive distributions. We find that the model misspecification
error dominates the variational approximation error in the \gls{VB}
posterior predictive distributions. These results explain the
empirical phenomenon that \gls{VB} predicts comparably well as
\gls{MCMC} even if it uses an approximating family. It also suggests
that we should focus on finding the correct model rather than
de-biasing the variational approximation if we use \gls{VB} for
prediction.

An interesting direction for future work is to characterize local
optima of the \gls{ELBO}, which is the \gls{VB} posterior we obtain in practice. The results in this work all assume that the
\gls{ELBO} optimization returns global optima. It provides the
possibility for local optima to share these properties, though further
research is needed to understand the precise properties of local
optima. Combining this work with optimization guarantees may lead to a
fruitful further characterization of variational Bayes.

%% file: sec_supp.tex
% !TEX root = vbmisspec.tex

\appendix
\onecolumn
{\Large\textbf{Supplementary Material: Variational Bayes under Model Misspecification}}

\section{The \gls{LAN} assumption (\Cref{assumption:paramlan}) in
parametric models}

\label{sec:paramlansatisfy}

The \gls{LAN} of parametric models (i.e. when parametric models
satisfy \Cref{assumption:paramlan}) has been widely studied in the
literature
\citep{van2000asymptotic,ghosal2017fundamentals,kleijn2012bernstein}.
For example, it suffices for a parametric model to satisfy: (1) the
log density $\log p(x\g \theta)$ is differentiable at $\theta^*$, (2)
the log likelihood ratio is bounded by some square integrable function
$m_{\theta^*}(x)$: $|\log \frac{p(x\g
\theta_1)}{p(x\g \theta_2)}|\leq m_{\theta^*}(x)
\left\|\theta_1-\theta_2\right\|$ $P_0$-almost surely, and (3) the
\gls{KL} divergence has a second order Taylor expansion around
$\theta^*$: $-\int p_0(x)\log \frac{p(x\g\theta)}{p(x\g \theta^*)} =
\frac{1}{2}(\theta-\theta^*)V_{\theta^*}(\theta-\theta^*) +
o(\left\|\theta-\theta^*\right\|)$ when $\theta\rightarrow \theta^*$
(Lemma 2.1 of \citet{kleijn2012bernstein}). Under these conditions,
the parametric model satisfies \Cref{assumption:paramlan} with
$\delta_n = (\sqrt{n})^{-1}$, which leads to the
$\sqrt{n}$-convergence of the exact posterior.

\section{Extending \Cref{thm:parammain,corollary:paramposteriorpred} to general probabilistic models}

\label{sec:extensions}

We study \gls{VB} in general probabilistic models via the above
reduction with the variational model $p^{\mathrm{VB}}(x\g\theta)$. We
posit analogous assumptions on $p^{\mathrm{VB}}(x\g\theta)$ as
\Cref{assumption:paramtest,assumption:paramlan} and extend the results
in \Cref{sec:misspecify} to general probabilistic models. Consider $n$
data points $x=x_{1:n}$, but they no longer need to be i.i.d. We
define the true data-generating density of $x$ as $p_0(x)$.

We state the asymptotic properties of the \gls{VB} posterior in
general probabilistic models.
\begin{thm}
\label{thm:genmain}
\emph{(Variational Bernstein--von Mises Theorem under model
misspecification)} Under
\Cref{assumption:paramprior,assumption:gentest,assumption:genlan},
\begin{enumerate}
\item The \gls{VB} posterior converges to a point mass at $\theta^*$:
\begin{align}
q^*(\theta)\stackrel{d}{\rightarrow}\delta_{\theta^*}.
\end{align}

\item Denote $\tilde{\theta} = \delta_n^{-1} (\theta - \theta^*)$ as
the re-centered and re-scaled version of $\theta$. The \gls{VB}
posterior of $\tilde{\theta}$ is asymptotically normal:
\begin{align}
  \left\|q^*({\tilde{\theta}}) - 
\cN({\tilde{\theta}} \s\Delta_{n,\theta^*},
    V'^{-1}_{\theta^*}))
   \right\|
  _{\mathrm{TV}}\stackrel{P_0}{\rightarrow} 0.
\end{align}
where $V'_{\theta^*}$ is diagonal and has the same diagonal terms as
$V_{\theta^*}$.
\item Denote $\hat{\theta}_{\mathrm{VB}} = \int \theta \cdot
q^*(\theta)\dif \theta$ as the mean of the \gls{VB} posterior. The
\gls{VB} posterior mean converges to $\theta^*$ almost surely:
\begin{align}
\hat{\theta}_{\mathrm{VB}}\stackrel{a.s.}{\rightarrow} \theta^*.
\end{align}
\item The \gls{VB} posterior mean is also asymptotically normal:
\begin{align}
\delta_n^{-1}(\hat{\theta}_{\mathrm{VB}} - \theta^*)\stackrel{d}{\rightarrow}\Delta_{\infty,\theta^*},
\end{align}
where $\Delta_{\infty,\theta^*}$ is the limiting distribution of the random vectors
$\Delta_{n,\theta^*}$: $\Delta_{n,\theta^*} \stackrel{d}{\rightarrow}
\Delta_{\infty,\theta^*}$. Its distribution is~${\Delta_{\infty,\theta^*} \sim
\cN\left(0, V_{\theta^*}^{-1}\E{P_0}{(\log
p^{\mathrm{VB}}(x\g\theta^*))'(\log
p^{\mathrm{VB}}(x\g\theta^*))'^\top}V_{\theta^*}^{-1}\right)}$.
\end{enumerate}
\end{thm}

\Cref{thm:genmain} repeats \Cref{thm:parammain} except that the
limiting distribution of the \gls{VB} posterior mean is governed by
$p^{\mathrm{VB}}(x\g\theta^*)$ as opposed to $p(x\g\theta^*)$.
\Cref{thm:genmain} reduces to \Cref{thm:parammain} when the
probabilistic model we fit is parametric.

With an additional \gls{LAN} assumption on the probabilistic model, we
can further extend the characterization of the \gls{VB} posterior
predictive distribution (\Cref{corollary:paramposteriorpred}) to
general probabilistic models.

\begin{assumption}[\gls{LAN}]
\label{assumption:genmodellan}
For every compact set $K\subset
    \mathbb{R}^d$, there exist random vectors $\Delta_{n,\theta^*}^0$ bounded
    in probability and nonsingular matrices $V_{\theta^*}^0$ such that
    \begin{align}
        \sup_{h\in K}
        \left|\log \frac{p(x\g {\theta^* + \delta_nh})}{p(x\g
        {\theta^*})} - h^\top V^0_{\theta^*}\Delta^0_{n,\theta^*}
        + \frac{1}{2}h^\top V^0_{\theta^*}h\right|
        \stackrel{P_0}{\rightarrow} 0.
    \end{align}
\end{assumption}
\Cref{assumption:genmodellan} requires that the probability model
$p(x\g \theta) = \int p(x, z\g \theta)\dif z$ has a \gls{LAN}
expansion at $\theta^*$. Many models satisfy
\Cref{assumption:genmodellan}, including Gaussian mixture models
\citep{westling2015establishing}, Poisson linear mixed models
\citep{hall2011asymptotic}, stochastic block models
\citep{bickel2013asymptotic}. We note that the \gls{LAN} expansion of
the model $p(x\g \theta)$ can be different from that of the
variational model $p^{\mathrm{VB}}(x\g\theta^*)$.

\begin{thm} 
\label{corollary:genposteriorpred}
(The \gls{VB} posterior predictive distribution) If the
probabilistic model is misspecified, i.e. $\left\|p_0(x) - p(x\g
{\theta^*})\right\|_{\mathrm{TV}} > 0,$ then the model approximation
error dominates the variational approximation error:
\begin{align}
\label{eq:paramposteriorpred2}
\frac{\left\|p^{\mathrm{pred}}_{\gls{VB}}(x_\mathrm{new}\g x) - p^{\mathrm{pred}}_{\mathrm{exact}}(x_\mathrm{new}\g x)\right\|_{\mathrm{TV}}}{\left\|p_0(x_\mathrm{new}) - p^{\mathrm{pred}}_{\mathrm{exact}}(x_\mathrm{new}\g x)\right\|_{\mathrm{TV}}}\stackrel{P_0}{\rightarrow} 0,
\end{align}
assuming $\int\nabla^2_\theta p(x\g \theta^*)\dif x < \infty$ and
\Cref{assumption:paramprior,assumption:gentest,assumption:genlan,assumption:genmodellan}.
Notation wise, $p^{\mathrm{pred}}_{\gls{VB}}(x_{\mathrm{new}}) = \int
p(x_{\mathrm{new}}\g \theta)q^*(\theta)\dif \theta$ is the \gls{VB}
posterior predictive density,
$p^{\mathrm{pred}}_{\mathrm{true}}(x_{\mathrm{new}}) = \int
p(x_{\mathrm{new}}\g \theta)p(\theta\g x)\dif \theta$ is the exact
posterior predictive density, $p(x\g \theta) = \int p(x, z\g
\theta)\dif z$ is the marginal density of the model, and $p_0(\cdot)$
is the true data generating density.
\end{thm}

\Cref{thm:genmain} and \Cref{corollary:genposteriorpred} generalizes
the asymptotic characterizations of the \gls{VB} posterior, the
\gls{VB} posterior mean, and the \gls{VB} posterior predictive
distributions to general probabilistic models. As in parametric
probabilistic models, the \gls{VB} posterior and its mean both remain
asymptotically normal and centered at $\theta^*$ in general
probabilistic models. The model misspecification error continues to
dominate the variational approximation error in the \gls{VB} posterior
predictive distributions.

The proofs of \Cref{thm:genmain} and \Cref{corollary:genposteriorpred}
extends those of \Cref{thm:parammain} and
\Cref{corollary:paramposteriorpred} by leveraging the connection
between the variational model and the original model
(\Cref{eq:generalvbred}). The full proofs are in
\Cref{sec:mainthmproof} and \Cref{sec:corollaryproof}.

\input{sec_app}

\section{Proof of \Cref{thm:genmain,thm:parammain}}

\label{sec:mainthmproof}

We prove \Cref{thm:genmain} in this section. \Cref{thm:parammain}
follows directly from \Cref{thm:genmain} in parametric models.

\begin{proof}
The proof of \Cref{thm:genmain} mimics the proof structure of
\citet{wang2018frequentist} except we need to take care of the
additional technical complications due to model misspecification.

We first study the \gls{VB} ideal $\pi^*(\theta\mid x)$, defined as
\begin{align*}
\pi^*(\theta\mid x) \stackrel{\Delta}{=} \frac{p(\theta)
p^{\mathrm{VB}}(x\s \theta)}{\int p(\theta) p^{\mathrm{VB}}(x\s
\theta)\dif \theta}.
\end{align*}
The \gls{VB} ideal $\pi^*(\theta\mid x)$ is the posterior of $\theta$
if we only perform variational approximation on the local latent
variables. In other words, it is the exact posterior of the
variational model $p^{\mathrm{VB}}(x\s
\theta)$.

We note that the \gls{VB} ideal is different from the
\gls{VB} posterior. However, we will show later that the two are
closely related.

Consider a change-of-variable (re-centering and re-scaling) of
$\theta$ into $\tilde{\theta} = \delta_n^{-1} (\theta - \theta^*)$.
\Cref{lemma:vbvm} shows that the posterior of $\tilde{\theta}$ is
close to $\cN(\cdot;\Delta_{n,\theta^*}, V^{-1}_{\theta^*})$, where
$\Delta_{n,\theta^*}$ and $V^{-1}_{\theta^*}$ are two constants in
\Cref{assumption:genlan}.
\begin{lemma}
\label{lemma:vbvm}
The \gls{VB} ideal converges in total variation to a sequence of normal
distributions,
\[
    ||
    \pi^*_{\tilde{\theta}}(\cdot\mid x)
    - 
    \cN(\cdot;\Delta_{n,\theta^*}, V^{-1}_{\theta^*})
    ||
    _{\gls{TV}}\stackrel{P_{0}}{\rightarrow} 0.\\
\]
\end{lemma}
The proof of \Cref{lemma:vbvm} is in \Cref{sec:vbvmlemmaproof}.
\Cref{lemma:vbvm} characterizes the posterior of the global latent
variables $\theta$ when we perform variational approximation on the
local latent variables $z$ under model misspecification.

Building on \Cref{lemma:vbvm}, \Cref{lemma:ivbconsist} and
\Cref{lemma:ivbnormal} characterize the \gls{KL} minimizer to the
\gls{VB} ideal within the mean field (i.e. factorizable) variational
family $\mathcal{Q}^d$. They pave the road for characterizing the
\gls{VB} posterior of the global latent variables $\theta$.

\begin{lemma}
\label{lemma:ivbconsist}
The \gls{KL} minimizer of the \gls{VB} ideal over the mean field
family is consistent: in $P_{0}$ probability, it converges to a
point mass centered at $\theta^*$,
\[
    \argmin_{q(\theta)\in\mathcal{Q}^d} 
    \gls{KL}(q(\theta)||\pi^*(\theta\mid x))
    \stackrel{d}{\rightarrow} \delta_{\theta^*}.\\
\]
\end{lemma}
The proof of \Cref{lemma:ivbconsist} is in \Cref{sec:ivbconsistproof}.
\Cref{lemma:ivbconsist} shows that the \gls{KL} minimizer to the
\gls{VB} ideal converges to a point mass centered at $\theta^*$.
\Cref{lemma:ivbconsist} is intuitive in that (1) the \gls{VB} ideal
converges to a point mass centered at $\theta^*$ and (2) the point
mass $\delta_{\theta^*}$ resides in the variational family
$\mathcal{Q}^d$.

\begin{lemma}
\label{lemma:ivbnormal}
The \gls{KL} minimizer of the \gls{VB} ideal of $\tilde{\theta}$
converges to that of $\cN(\cdot
\s\Delta_{n,\theta^*}, V^{-1}_{\theta^*})$ in total variation: under
mild technical conditions on the tail behavior of $\mathcal{Q}^d$
(see \Cref{assumption:scoreint} in \Cref{sec:ivbnormalproof}),
\[
    \left\|
    \argmin_{q\in\mathcal{Q}^d} 
    \gls{KL}(q(\cdot)||\pi^*_{\tilde{\theta}}(\cdot\mid x))
    -
    \argmin_{q\in\mathcal{Q}^d} 
    \gls{KL}(q(\cdot)|| \cN(\cdot \s\Delta_{n,\theta^*}, V^{-1}_{\theta^*}))
    \right\|
    _{\gls{TV}}\stackrel{P_0}{\rightarrow} 0.\\
\]
\end{lemma}
The proof of \Cref{lemma:ivbnormal} is in \Cref{sec:ivbnormalproof}.
\Cref{lemma:ivbnormal} shows that the \gls{KL} minimizer to the
\gls{VB} ideal converges to the \gls{KL} minimizer to
$\cN(\cdot;\Delta_{n,\theta^*}, V^{-1}_{\theta^*})$. As with
\Cref{lemma:ivbconsist}, \Cref{lemma:ivbnormal} is also intuitive
because \Cref{lemma:vbvm} has shown that
$\pi^*_{\tilde{\theta}}(\cdot\mid x)$ and $\cN(\cdot
\s\Delta_{n,\theta^*}, V^{-1}_{\theta^*})$ are close in the large
sample limit.

The final step of the proof is to establish the connection between the
\gls{VB} posterior and the \gls{KL} minimizer of the \gls{VB} ideal.
First notice that the \gls{VB} posterior is then the minimizer of the so-called profiled
\gls{ELBO}:
\begin{align}
\label{eq:vbpost_profile}
q^*(\theta) = \argmax_{q(\theta)}\gls{ELBO}_p(q(\theta)),
\end{align}
which treats the variational posterior of local latent variables $z$'s
as a function of $q(\theta)$. Technically, the profiled \gls{ELBO} is
defined as follows:
\begin{align}
    \label{eq:vb_post}
    \gls{ELBO}_p(q(\theta)):=\sup_{q(z)}
    \int q(\theta)
    \left(
    \log \left[p(\theta)
    \exp\left\{
    \int q(z) \log \frac{p(x, z\mid \theta)}{q(z)}\dif z
    \right\}\right]
    - \log q(\theta) \right) \dif \theta.
\end{align}
Via this representation of the \gls{VB} posterior, Lemma 4 of
\citet{wang2018frequentist} shows that the
\gls{VB} posterior and the \gls{KL} minimizer of the \gls{VB} ideal
are close in the large sample limit. We restate this result here for
completeness.
\begin{lemma}[Lemma 4 of \citet{wang2018frequentist}] The negative
\gls{KL} divergence to the \gls{VB} ideal is equivalent to the
profiled \gls{ELBO} in the limit: under mild technical conditions on
the tail behavior of $\mathcal{Q}^d$, for $q(\theta)\in
\mathcal{Q}^d,$
\label{lemma:vb_post}
\[
    \gls{ELBO}_p(q(\theta)) = 
    -\gls{KL}(q(\theta) || \pi^*(\theta\mid x)) + o_{P_0}(1).
\]
\end{lemma}

Given
\Cref{lemma:vbvm,lemma:ivbconsist,lemma:ivbnormal,lemma:vb_post}, we
can prove \Cref{thm:genmain}. 

\Cref{thm:genmain}.1 and \Cref{thm:genmain}.2 are direct consequences of
\Cref{lemma:ivbconsist,lemma:ivbnormal,lemma:vb_post}. We have
\[
    \left\|
    \argmax_{q\in\mathcal{Q}^d} 
    \gls{ELBO}_p(q(\tilde{\theta}))
    -
    \argmin_{q\in\mathcal{Q}^d} 
    \gls{KL}(q(\cdot)|| \cN(\cdot \s\Delta_{n,\theta^*}, V^{-1}_{\theta^*}))
    \right\|
    _{\gls{TV}}\stackrel{P_0}{\rightarrow} 0,
\]
which leads to the consistency and asymptotic normality of
$q^*(\theta)$ due to \Cref{eq:vbpost_profile}.

\Cref{thm:genmain}.3 and \Cref{thm:genmain}.4 follows from
\Cref{lemma:ivbconsist,lemma:ivbnormal,lemma:vb_post} via a similar
proof argument with Theorem 2.3 in \citet{kleijn2012bernstein} and
Theorem 10.8 in \citet{van2000asymptotic}.

We consider three stochastic processes: fix some compact set $K$
and for given $M>0$,
\begin{align}
t&\mapsto Z_{n,M}(t) = \int_{||\tilde{\theta}||\leq M} 
(t-\tilde{\theta})^2\cdot q^*_{\tilde{\theta}}(\tilde{\theta})
\dif \tilde{\theta},\\
t&\mapsto W_{n,M}(t) = \int_{||\tilde{\theta}||\leq M} 
(t-\tilde{\theta})^2\cdot \cN(\tilde{\theta};\Delta_{n,\theta^*}, 
(V_{\theta^*}^{'})^{-1})\dif \tilde{\theta},\\
t&\mapsto W_{M}(t) = \int_{||\tilde{\theta}||\leq M} 
(t-\tilde{\theta})^2\cdot \cN(\tilde{\theta};X, (V_{\theta^*}^{'})^{-1})
\dif \tilde{\theta}.
\end{align}

The intuition behind these constructions is that (1)
$\tilde{\theta}^{\mathrm{VB}} =
\delta_n^{-1}(\hat{\theta}_{\mathrm{VB}} - \theta^*)$ is the minimizer
of the process $t\mapsto Z_{n,\infty}(t)$ and (2) $X = \int \tilde{\theta}
\cdot \cN(\tilde{\theta} \s X, V^{'-1}_{\theta^*} )
\dif\tilde{\theta}$ is the minimizer of $t\mapsto W_{\infty}(t)$.

To prove \Cref{thm:genmain}.3 and \Cref{thm:genmain}.4, we have
\[Z_{n,M}-W_{n,M}=o_{P_0}(1) \text{ in } \ell^\infty(K)\] due to
\Cref{thm:genmain}.2 and $\sup_{t\in K, ||h||\leq M}(t-h)^2<\infty$. Then
we have \[W_{n,M} - W_{M}=o_{P_0}(1) \text{ in } \ell^\infty(K)\] due
to $\Delta_{n,\theta^*}\stackrel{d}{\rightarrow}X$ and the continuous
mapping theorem. Finally, we have
\[W_M - W_{\infty} = o_{P_0}(1) \text{ in } \ell^\infty(K)\]
as $M\rightarrow\infty$ because of $\int \theta \cdot
q^*(\theta)<\infty$, and
\[Z_{n,p^{\mathrm{VB}}} - Z_{n,\infty} = o_{P_0}(1) \text{ in } \ell^\infty(K)\]
due to $\int_{||\tilde{\theta}||>p^{\mathrm{VB}}}||\tilde{\theta}||^2
q^*(\tilde{\theta})\dif\tilde{\theta}
\stackrel{P_0}{\rightarrow}0$ for for any $p^{\mathrm{VB}}\rightarrow\infty$
ensured by \Cref{assumption:vbpost_tech}.1. Therefore, we have
\[Z_{n,\infty} - W_\infty = o_{P_0}(1) \text{ in } \ell^\infty(K),\]
which implies 
\[\tilde{\theta}^{\mathrm{VB}}\stackrel{d}{\rightarrow}X\]
due to the continuity and convexity of the squared loss and the argmax
theorem.

\end{proof}

\section{Proof of \Cref{corollary:genposteriorpred,corollary:paramposteriorpred}}

\label{sec:corollaryproof}

We prove \Cref{corollary:genposteriorpred} here.
\Cref{corollary:paramposteriorpred} is a direct consequence of
\Cref{corollary:genposteriorpred}.

\begin{proof}

We next study the posterior predictive distribution resulting from the
\gls{VB} posterior.
For notational simplicity, we abbreviate
$p^{\mathrm{pred}}_{\gls{VB}}(x_\mathrm{new}\g x_{1:n})$ as
$p^{\mathrm{pred}}_{\gls{VB}}(x_\mathrm{new})$.

\begin{align}
&\left\|p^{\mathrm{pred}}_{\gls{VB}}(x_\mathrm{new}) - p^{\mathrm{pred}}_{\mathrm{true}}(x_\mathrm{new})\right\|_{\mathrm{TV}}\\
=&\left\|\int
p(x_{\mathrm{new}}\g \theta)q^*(\theta)\dif \theta - \int
p(x_{\mathrm{new}}\g \theta)p(\theta\g x)\dif \theta\right\|_{\mathrm{TV}}\\
=&\frac{1}{2}\int \left|\int p(x_{\mathrm{new}}\g \theta)q^*(\theta)\dif \theta - \int
p(x_{\mathrm{new}}\g \theta)p(\theta\g x)\dif \theta\right|\dif x_\mathrm{new}\\
=&\frac{1}{2}\int \left|\int p(x_{\mathrm{new}}\g \theta)\left(q^*(\theta)-p(\theta\g x)\right)\dif \theta\right|\dif x_\mathrm{new}\\
\leq & \frac{1}{2}\int \left|\int p(x_\mathrm{new}\g \theta^*)\left(q^*(\theta)-p(\theta\g x)\right)\dif \theta\right|\dif x_\mathrm{new} \\
&+ \frac{1}{2}\int\left| \int\left(p(x_{\mathrm{new}}\g \theta) - p(x_\mathrm{new}\g \theta^*)\right)\cdot\left(q^*(\theta)-p(\theta\g x)\right)\dif \theta\right|\dif x_\mathrm{new}\label{eq:finalterm}
\end{align}

The first equality is due to the definition of posterior predictive
densities. The second equality is due to the definition of the
\glsreset{TV}\gls{TV} distance. The third equality collects the two
integrals into one. The fourth equality is due to
$p_\theta(x_\mathrm{new}) \geq 0$ and triangle inequality. 

If each term in \Cref{eq:finalterm} goes to zero in the large sample
limit, we have
\begin{align}
\label{eq:numerator}
\left\|p^{\mathrm{pred}}_{\gls{VB}}(x_\mathrm{new}) - p^{\mathrm{pred}}_{\mathrm{true}}(x_\mathrm{new})\right\|_{\mathrm{TV}}\stackrel{P_0}{\rightarrow} 0.
\end{align}

Moreover, we assume the model $\{p_{\theta}:\theta\in\Theta\}$ is
misspecified, which implies
\begin{align}
\label{eq:denominator}
&\left\|p_0(x_\mathrm{new}) - p^{\mathrm{pred}}_{\mathrm{true}}(x_\mathrm{new})\right\|_{\mathrm{TV}}\\
\geq&\left\|p_0(x_\mathrm{new}) - p(x_\mathrm{new}\g \theta^*)\right\|_{\mathrm{TV}} - \left\|p^{\mathrm{pred}}_{\mathrm{true}}(x_\mathrm{new}) - p(x_\mathrm{new}\g \theta^*)\right\|_{\mathrm{TV}}\\
\stackrel{P_0}{\rightarrow} &\left\|p_0(x_\mathrm{new}) - p(x_\mathrm{new}\g \theta^*)\right\|_{\mathrm{TV}} \\
 > & 0.
\end{align}
The first inequality is due to triangle inequality. The second
equation is due to a similar argument with
$\left\|p^{\mathrm{pred}}_{\gls{VB}}(x_\mathrm{new}) -
p^{\mathrm{pred}}_{\mathrm{true}}(x_\mathrm{new})\right\|_{\mathrm{TV}}\stackrel{P_0}{\rightarrow}
0$. The intuition is that the posterior $p(\theta\g
x)\stackrel{P_0}{\rightarrow}
\delta_{\theta^*}$ in the large sample limit, so the posterior
predictive distribution should converge to
$p(x_\mathrm{new}\g \theta^*) = \int p_\theta(x_\mathrm{new})
\delta_{\theta^*}(\theta) \dif \theta$. The last step
$\left\|p_0(x_\mathrm{new}) -
p(x_\mathrm{new}\g \theta^*)\right\|_{\mathrm{TV}} > 0$ is due to the
assumption that the model $p(\cdot\g \theta)$ is misspecified.

\Cref{eq:numerator} and \Cref{eq:denominator} together imply
\Cref{corollary:genposteriorpred}:
\begin{align*}
\frac{\left\|p^{\mathrm{pred}}_{\gls{VB}}(x_\mathrm{new}) - p^{\mathrm{pred}}_{\mathrm{true}}(x_\mathrm{new})\right\|_{\mathrm{TV}}}{\left\|p_0(x_\mathrm{new}) - p^{\mathrm{pred}}_{\mathrm{true}}(x_\mathrm{new})\right\|_{\mathrm{TV}}}\stackrel{P_0}{\rightarrow} 0.
\end{align*}

Below we show that each term in \Cref{eq:finalterm} goes to zero in
the large sample limit, which completes the proof.

For the first term in \Cref{eq:finalterm}, we have.
\begin{align*}
&\int \left|\int p(x_\mathrm{new}\g \theta^*)\left(q^*(\theta)-p(\theta\g x)\right)\dif \theta\right|\dif x_\mathrm{new} \\
=&\int \left|p(x_\mathrm{new}\g \theta^*)\int \left(q^*(\theta)-p(\theta\g x)\right)\dif \theta\right|\dif x_\mathrm{new} \\
=&\int \left|p(x_\mathrm{new}\g \theta^*)\cdot 0\right|\dif x_\mathrm{new} \\
=&0.
\end{align*}
The first equality is due to $p(x_\mathrm{new}\g \theta^*)$ not
depending on $\theta$. The second equality is due to both
$q^*(\theta)$ and $p(\theta\g x)$ being probability density functions.
The third equality is due to integration of zero equal to zero.

Next we note that
\begin{align*}
&\int\left| \int\left(p(x_{\mathrm{new}}\g \theta) - p(x_\mathrm{new}\g \theta^*)\right)\cdot\left(q^*(\theta)-p(\theta\g x)\right)\dif \theta\right|\dif x_\mathrm{new} \\
\leq & \int\left| \int\left(p(x_{\mathrm{new}}\g \theta) - p(x_\mathrm{new}\g \theta^*)\right)\cdot\left(q^*(\theta)-\cN(\theta\s \theta^*, \delta_n^\top V_{\theta^*}^{'-1}\delta_n)\right)\dif \theta\right|\dif x_\mathrm{new} \\
&+ \int\left| \int\left(p(x_{\mathrm{new}}\g \theta) - p(x_\mathrm{new}\g \theta^*)\right)\cdot\left(\cN(\theta\s \theta^*, \delta_n^\top V_{\theta^*}^{'-1}\delta_n)-\cN(\theta\s \theta^*, \delta_n^\top V_{\theta^*}^{-1}\delta_n)\right)\dif \theta\right|\dif x_\mathrm{new} \\
&+ \int\left| \int\left(p(x_{\mathrm{new}}\g \theta) - p(x_\mathrm{new}\g \theta^*)\right)\cdot\left(\cN(\theta\s \theta^*, \delta_n^\top V_{\theta^*}^{-1}\delta_n)-p(\theta\g x)\right)\dif \theta\right|\dif x_\mathrm{new}.
\end{align*}
We apply the Taylor's theorem to $p(x_{\mathrm{new}}\g \theta) -
p(x_\mathrm{new}\g \theta^*)$: There exists some function
$h_{\theta^*}(\theta)$ such that
\begin{align*}
&p(x_{\mathrm{new}}\g \theta) - p(x_\mathrm{new}\g \theta^*) \\
=& (\theta-\theta^*)\cdot\nabla_\theta p(x_{\mathrm{new}}\g \theta)\big|_{\theta=\theta^*}\\
& + \nabla^2_\theta p(x_{\mathrm{new}}\g \theta)\big|_{\theta=\theta^*}\cdot(\theta-\theta^*)(\theta-\theta^*)^\top \\
& + h_{\theta^*}(\theta)\cdot(\theta-\theta^*)(\theta-\theta^*)^\top,
\end{align*}
where $\lim_{\theta\rightarrow\theta^*}h_{\theta^*}(\theta) = 0.$ We
apply this expansion to each of the term above:
\begin{align*}
&\int\left| \int\left(p(x_{\mathrm{new}}\g \theta) - p(x_\mathrm{new}\g \theta^*)\right)\cdot\left(q^*(\theta)-\cN(\theta\s \theta^*, \delta_n^\top V_{\theta^*}^{'-1}\delta_n)\right)\dif \theta\right|\dif x_\mathrm{new}\\
\leq&\int\left| \int\left( (\theta-\theta^*)\cdot\nabla_\theta p(x_{\mathrm{new}}\g \theta)\big|_{\theta=\theta^*} \right)\cdot\left(q^*(\theta)-\cN(\theta\s \theta^*, \delta_n^\top V_{\theta^*}^{'-1}\delta_n)\right)\dif \theta\right|\dif x_\mathrm{new}\\
&+\int\left| \int\left( \nabla^2_\theta p(x_{\mathrm{new}}\g \theta)\big|_{\theta=\theta^*}\cdot(\theta-\theta^*)(\theta-\theta^*)^\top  \right)\cdot\left(q^*(\theta)-\cN(\theta\s \theta^*, \delta_n^\top V_{\theta^*}^{'-1}\delta_n)\right)\dif \theta\right|\dif x_\mathrm{new}\\
&+\int\left| \int\left( h_{\theta^*}(\theta)\cdot(\theta-\theta^*)(\theta-\theta^*)^\top  \right)\cdot\left(q^*(\theta)-\cN(\theta\s \theta^*, \delta_n^\top V_{\theta^*}^{'-1}\delta_n)\right)\dif \theta\right|\dif x_\mathrm{new}\\
\rightarrow & 0\cdot \int \left| \nabla_\theta p(x_{\mathrm{new}}\g \theta)\big|_{\theta=\theta^*}\right|\dif x_\mathrm{new} + 0\cdot \int \left| \nabla^2_\theta p(x_{\mathrm{new}}\g \theta)\big|_{\theta=\theta^*}\right|\dif x_\mathrm{new} + 0\cdot \int \left|  h_{\theta^*}(\theta)\right|\dif x_\mathrm{new} \\
=& 0
\end{align*}
The key property that enables the calculation above is that
$q^*(\theta)$ and $\cN(\theta\s \theta^*, \delta_n^\top
V_{\theta^*}^{'-1}\delta_n)$ share the same first and second moments.

With the same argument, we can show that 
\begin{align*}
\int\left| \int\left(p(x_{\mathrm{new}}\g \theta) - p(x_\mathrm{new}\g \theta^*)\right)\cdot\left(\cN(\theta\s \theta^*, \delta_n^\top V_{\theta^*}^{-1}\delta_n)-p(\theta\g x)\right)\dif \theta\right|\dif x_\mathrm{new}\rightarrow 0.
\end{align*}

Finally, we work with the middle term.
\begin{align*}
& \int\left| \int\left(p(x_{\mathrm{new}}\g \theta) - p(x_\mathrm{new}\g \theta^*)\right)\cdot\left(\cN(\theta\s \theta^*, \delta_n^\top V_{\theta^*}^{'-1}\delta_n)-\cN(\theta\s \theta^*, \delta_n^\top V_{\theta^*}^{-1}\delta_n)\right)\dif \theta\right|\dif x_\mathrm{new} \\
\leq &\int\left| \int\left((\theta-\theta^*)\cdot\nabla_\theta p(x_{\mathrm{new}}\g \theta)\big|_{\theta=\theta^*}\right)\cdot\left(\cN(\theta\s \theta^*, \delta_n^\top V_{\theta^*}^{'-1}\delta_n)-\cN(\theta\s \theta^*, \delta_n^\top V_{\theta^*}^{-1}\delta_n)\right)\dif \theta\right|\dif x_\mathrm{new} \\
&+ \int\left| \int\left(\nabla^2_\theta p(x_{\mathrm{new}}\g \theta)\big|_{\theta=\theta^*}\cdot(\theta-\theta^*)(\theta-\theta^*)^\top\right)\right.\\
&\left.\cdot\left(\cN(\theta\s \theta^*, \delta_n^\top V_{\theta^*}^{'-1}\delta_n)-\cN(\theta\s \theta^*, \delta_n^\top V_{\theta^*}^{-1}\delta_n)\right)\dif \theta\right|\dif x_\mathrm{new} \\
&+ \int\left| \int\left(h_{\theta^*}(\theta)\cdot(\theta-\theta^*)(\theta-\theta^*)^\top\right)\cdot\left(\cN(\theta\s \theta^*, \delta_n^\top V_{\theta^*}^{'-1}\delta_n)-\cN(\theta\s \theta^*, \delta_n^\top V_{\theta^*}^{-1}\delta_n)\right)\dif \theta\right|\dif x_\mathrm{new} \\
=&0\cdot \int \left| \nabla_\theta p(x_{\mathrm{new}}\g \theta)\big|_{\theta=\theta^*}\right|\dif x_\mathrm{new} \\
&+ (\delta_n^\top V_{\theta^*}^{-1}\delta_n-\delta_n^\top V_{\theta^*}^{'-1}\delta_n)\cdot \int \left| \nabla^2_\theta p(x_{\mathrm{new}}\g \theta)\big|_{\theta=\theta^*}\right|\dif x_\mathrm{new} \\
&+ (\delta_n^\top V_{\theta^*}^{-1}\delta_n-\delta_n^\top V_{\theta^*}^{'-1}\delta_n)\cdot \int \left|  h_{\theta^*}(\theta)\right|\dif x_\mathrm{new} \\
\rightarrow &0
\end{align*}
The last step is because $\delta_n\rightarrow 0$ and $\int
\nabla^2_\theta p(x_{\mathrm{new}}\g
\theta)\big|_{\theta=\theta^*}\dif x_\mathrm{new} < \infty$.

\end{proof}

\section{Proof of \Cref{lemma:vbvm}}
\label{sec:vbvmlemmaproof}

In this proof, we only need to show that \Cref{assumption:gentest}
implies Assumption (2.3) in \citet{kleijn2012bernstein}:
$\int_{\tilde{\theta}>p^{\mathrm{VB}}} \pi^*_{\tilde{\theta}}(\tilde{\theta}\mid
x) \dif \tilde{\theta}\stackrel{P_0}{\rightarrow} 0$ for every
sequence of constants $p^{\mathrm{VB}} \rightarrow \infty$, where
$\tilde{\theta} = \delta_n^{-1}(\theta-\theta^*).$

To prove this implication, we repeat Theorem 3.1, Theorem 3.3, Lemma
3.3, Lemma 3.4 of \citet{kleijn2012bernstein}. The only difference is
that we prove it for the general convergence $\delta_n$ instead of the
parametric convergence rate $\sqrt{n}$. The idea is to consider test
sequences of uniform exponential power around $\theta^*$. We omit the
proof here; see \citet{kleijn2012bernstein} for details.

This proof also resembles the proof of Lemma 1 in
\citet{wang2018frequentist}.

\section{Proof of \Cref{lemma:ivbconsist}}

\label{sec:ivbconsistproof}

We focus on the \gls{VB} posterior of $\theta$ which converges with
the $\delta_n$ rate. Without loss of generality, we consider the
subset of mean field variational family that also shrinks with the
rate $\delta_n$. The rationale of this step is that the \gls{KL}
divergence between exact posterior and the \gls{VB} posterior will
blow up to $\infty$ for other classes of variational families. More
precisely, we assume the following variational family $\mathcal{Q}$
\begin{align}
  q_{\check{\theta}}(\check{\theta}) = q(\mu + \delta_n
  \check{\theta}) | \textrm{det}(\delta_n) |,
\end{align}
where $\check{\theta}:=
\delta_n^{-1}(\theta - \mu)$, for some $\mu \in \Theta$.

Note the variational family is allowed to center at any value, not
necessarily at $\theta^*$. 

We now characterize the limiting distribution of the \gls{KL}
minimizer of the \gls{VB} ideal. In other words, the mass of the
\gls{KL} minimizer concentrates near $\theta^*$ as $n\rightarrow
\infty$:
\[q^\ddagger(\theta) :=
\argmin_{q(\theta)\in\mathcal{Q}^d}
        \gls{KL}(q(\theta)||\pi^*(\theta\mid x))  \stackrel{d}{\rightarrow} \delta_{\theta^*}.\]
It suffices to show
$\int_{B(\theta^*, \xi_n)} q^\ddagger (\theta)\dif 
\theta\stackrel{P_0}{\rightarrow} 1,$
for some $\xi_n\rightarrow 0$ as $n\rightarrow \infty$ due to the
Slutsky's theorem.

The proof below mimics the proof of Lemma 2 of
\citep{wang2018frequentist} (also the Step 2 in the proof of Lemma 3.6
along with Lemma 3.7 in \citet{lu2016gaussapprox}) except we take care
of the extra technicality due to model misspecification. We include
the proof for completeness here.

We start with two claims that we will prove later.
\begin{align}
\label{claim:upbdkl}
\limsup_{n\rightarrow\infty}\min\gls{KL}
(q(\theta)||\pi^*(\theta\mid x)) \leq M,
\end{align}
\begin{align}
\label{claim:compact}
\int_{\mathbb{R}^d\backslash K}q^{\ddagger}(\theta)\dif 
\theta\stackrel{}{\rightarrow}0,
\end{align}
where $M >0$ is some constant and $K$ is the compact set assumed in
the local asymptotic normality condition. We will use them to upper
bound and lower bound $\int_{B(\theta^*, \xi_n)} q^{\ddagger, K}
(\theta)\dif
\theta$.

The upper bound of $\int_{B(\theta^*, \xi_n)} q^{\ddagger, K}
(\theta)\dif
\theta$ is due to the \gls{LAN} condition,
\begin{align*}
&\int q^{\ddagger,K}(\theta)p^{\mathrm{VB}}(\theta\s x)\dif \theta\\
=&\int q^{\ddagger,K}(\theta)\left[p^{\mathrm{VB}}(\theta^*\s x)+
\delta_n^{-1}(\theta-\theta^*)^\top 
V_{\theta^*}\Delta_{n,\theta^*} \right.\\
&\left.- 
\frac{1}{2}[\delta_n^{-1}(\theta-\theta^*)]^\top 
V_{\theta^*}[\delta_n^{-1}(\theta-\theta^*)]+o_P(1)\right]\dif \theta\\
\leq& p^{\mathrm{VB}}(\theta^*\s x) - C_1\sum_{i=1}^d\frac{\eta^2}{\delta_{n,ii}^2}
\int_{B(\theta^*, \eta)^c} q^{\ddagger,K}(\theta)\dif \theta +o_P(1),
\label{eq:klupperbd}
\end{align*}
for large enough $n$ and $\eta << 1$ and some constant $C_1 > 0$. 

The lower bound of the integral is due to the first claim:
\begin{align}
&\int q^{\ddagger,K}(\theta)p^{\mathrm{VB}}(\theta\s x)
\dif \theta\geq p^{\mathrm{VB}}(\theta^*\s x)-M_0,
\end{align}
for some large constant $M_0>M$. This is due to two steps. First, Eq.
31 of \citet{wang2018frequentist} gives
\begin{align}
&\gls{KL}(q^{\ddagger,K}(\theta)||\pi^*(\theta\mid x))\\
=& \log |\det(\delta_n)|^{-1} + 
\sum^d_{i=1} \mathbb{H}(q_{h,i}^{\ddagger,K}(h)) 
- \int q^{\ddagger,K}(\theta)\log \pi^*(\theta\mid x)\dif \theta.
\label{eq:rewritekl}
\end{align}
Then we approximate the last term by the \gls{LAN} condition:
\begin{align}
&\int q^{\ddagger,K}(\theta)\log \pi^*(\theta\mid x)\dif \theta\\
=&\int q^{\ddagger,K}(\theta)\log p(\theta)\dif \theta + 
\int q(\theta)p^{\mathrm{VB}}(\theta\s x)\dif \theta 
-\log \int p(\theta)\exp(p^{\mathrm{VB}}(\theta\s x))\dif \theta\\
=& \int q^{\ddagger,K}(\theta)\log p(\theta)\dif \theta + 
\int q^{\ddagger,K}(\theta)p^{\mathrm{VB}}(\theta\s x)\dif \theta \nonumber\\
&- \left[\frac{d}{2}\log(2\pi) - \frac{1}{2}\log \det V_{\theta^*}+
\log\det(\delta_n)+p^{\mathrm{VB}}(\theta^*\s x)+\log p(\theta^*)+o_P(1)\right].
\label{eq:normalizerapprox}
\end{align}

The above approximation leads to the following approximation to the
\gls{KL} divergence:
\begin{align}
&\gls{KL}(q^{\ddagger,K}(\theta)||\pi^*(\theta\mid x))\\
=& \log |\det(\delta_n)|^{-1} + 
\sum^d_{i=1} \mathbb{H}(q_{h,i}^{\ddagger,K}(h)) -  
\int q^{\ddagger,K}(\theta)\log p(\theta)\dif \theta - 
\int q^{\ddagger,K}(\theta)p^{\mathrm{VB}}(\theta\s x)\dif \theta \nonumber\\
&+\left[\frac{d}{2}\log(2\pi) - \frac{1}{2}\log \det V_{\theta^*}+
\log\det(\delta_n)+p^{\mathrm{VB}}(\theta^*\s x)+
\log p(\theta^*)+o_P(1)\right]\label{eq:canceldelta1}\\
=&\sum^d_{i=1} \mathbb{H}(q_{h,i}^{\ddagger,K}(h)) -  
\int q^{\ddagger,K}(\theta)\log p(\theta)\dif \theta - 
\int q^{\ddagger,K}(\theta)p^{\mathrm{VB}}(\theta\s x)\dif \theta\nonumber\\
&+\frac{d}{2}\log(2\pi) - \frac{1}{2}\log \det V_{\theta^*}+
p^{\mathrm{VB}}(\theta^*\s x)+\log p(\theta^*)+o_P(1).
\label{eq:klapprox}
\end{align}

Then via the first claim above, we have
\begin{align}
&\int q^{\ddagger,K}(\theta)p^{\mathrm{VB}}(\theta\s x)\dif \theta\\
\geq &-M +\sum^d_{i=1} \mathbb{H}(q_{h,i}^{\ddagger,K}(h)) -  
\int q^{\ddagger,K}(\theta)\log p(\theta)\dif \theta \nonumber\\
&+ \frac{d}{2}\log(2\pi) - \frac{1}{2}\log \det V_{\theta^*}+
p^{\mathrm{VB}}(\theta^*\s x)+\log p(\theta^*)+o_P(1)\\
\geq &-M_0 + p^{\mathrm{VB}}(\theta^*\s x) + o_P(1)
\label{eq:kllowerbd}
\end{align}
for some constant $M_0>0$. The last step is because the only term that
depends on $n$ is $\int q^{\ddagger,K}(\theta)\log p(\theta)\dif
\theta$ which is finite due to \Cref{assumption:paramprior}.

Combining the lower and upper bounds of the integral gives
\begin{align*}
&p^{\mathrm{VB}}(\theta^*\s x) - C_1\sum_{i=1}^d\frac{\eta^2}{\delta_{n,ii}^2}
\int_{B(\theta^*, \eta)^c} q^{\ddagger,K}(\theta)\dif \theta +o_P(1)
\geq &-M_0 + p^{\mathrm{VB}}(\theta^*\s x)\\
\Rightarrow&\int_{B(\theta^*, \eta)^c} q^{\ddagger,K}(\theta)\dif \theta +o_P(1)
\leq \frac{M_0\cdot (\min_i\delta_{n,ii})^2}{C_2\eta^2},
\end{align*}
for some constant $C_2>0$. By choosing $\eta =
\sqrt{M_0(\min_i\delta_{n,ii})/C_2}\rightarrow 0$, this term go
to zero as $n$ goes to infinity. In other words, we have shown
$\int_{B(\theta^*, \xi_n)} q^\ddagger (\theta)\dif
\theta\stackrel{P_0}{\rightarrow} 1$ with $\xi_n = \eta$.

We now prove the two claims made at the beginning. To show
\Cref{claim:upbdkl}, it suffices to show that there exists a choice of
$q(\theta)$ such that
\begin{align*}
\limsup_{n\rightarrow\infty}\gls{KL}(q(\theta)||\pi^*(\theta\mid x)) 
<\infty.
\end{align*}
We choose $\tilde{q}(\theta) = \prod_{i=1}^d N(\theta_i;\theta_{0,i}, 
\delta^2_{n,ii}v_i)$ for $v_i>0, i = 1, ..., d$. We thus have
\begin{align}
&\gls{KL}(\tilde{q}(\theta)||\pi^*(\theta\mid x))\\
=&\sum^d_{i=1} \frac{1}{2}\log(v_i)+\frac{d}{2} + d\log(2\pi)-  
\int \tilde{q}(\theta)\log p(\theta)\dif \theta - 
\int \tilde{q}(\theta)p^{\mathrm{VB}}(\theta\s x)\dif \theta\nonumber\\
& - \frac{1}{2}\log \det V_{\theta^*}+p^{\mathrm{VB}}(\theta^*\s x)+
\log p(\theta^*)+o_P(1)\\
=&\sum^d_{i=1} \frac{1}{2}\log(v_i)+\frac{d}{2} + d\log(2\pi)
- \frac{1}{2}\log \det V_{\theta^*}+C_6+o_P(1),
\label{eq:relaxcompact}
\end{align}
for some constant $C_6 > 0.$ The finiteness of limsup is due to the
boundedness of the last term. The second equality is due to the limit
of $\tilde{q}(\theta)$ concentrating around $\theta^*$. Specifically,
we expand $\log p(\theta)$ to the second order around $\theta^*$,
\begin{align*}
&\int \tilde{q}(\theta)\log p(\theta)\dif \theta\\
=& \log p(\theta^*) + \int \tilde{q}(\theta)\left[(\theta-\theta^*)(\log p(\theta^*))'+
\frac{(\theta-\theta^*)^2}{2}
\int_0^1(\log p(\xi\theta+(1-\xi)\theta^*))''(1-\xi)^2\dif \xi\right]
\dif \theta\\
\leq & \log p(\theta^*) + 
\frac{1}{2!}\max_{\xi\in[0,1]}
\left\{\int \tilde{q}(\theta)(\theta-\theta^*)^2
(\log p(\xi\theta+(1-\xi)\theta^*))''
\dif \theta\right\}\\
\leq&\log p(\theta^*) +\frac{M_p}{\sqrt{(2\pi)^d
\det(\delta_n^2)\prod_iv_i}}\int_{\mathbb{R}^d}|\theta|^2e^{(|\theta|+
|\theta^*|)^2}\cdot e^{-\frac{1}{2}
\theta^\top(\delta_nV\delta_n)^{-1}\theta}
\dif \theta\\
\leq & \log p(\theta^*) +
\frac{M_p}{\sqrt{(2\pi)^d\det(\delta_n^2)\prod_iv_i}} e^{\theta^{*2}}
\int_{\mathbb{R}^d}|\theta|^2e^{-\frac{1}{2}\theta^\top
[(\delta_nV\delta_n)^{-1}-2I_d]\theta}\\
\leq & \log p(\theta^*) + C_3M_pe^{\theta^{*2}}\max_d(\delta_{n,ii}^2)
\det(V^{-1}-2\delta_n^2)^{-1}\\
\leq & \log p(\theta^*) + C_4\max_d(\delta_{n,ii}^2)
\end{align*}
where $\max_d(\delta_{n,ii}^2)\rightarrow 0$ and $C_3, C_4 >0$. The
first two equalities are due to Taylor expansion. The third inequality
is due to the tail condition in \Cref{assumption:paramprior}.  The
fourth and fifth are due to rescaling $\theta$ appealing to the mean
of a Chi-squared distribution with $d$ degrees of freedom. The last
inequality is due to $\det(V^{-1}-2\delta_n^2)^{-1} > 0$ for large
enough $n$.

We apply the same Taylor expansion argument to the $\int
\tilde{q}(\theta)p^{\mathrm{VB}}(\theta\s x)\dif \theta$ leveraging
the \gls{LAN} condition
\begin{align*}
&\int_{K_n} \tilde{q}(\theta)p^{\mathrm{VB}}(\theta\s x)\dif \theta\\
=& p^{\mathrm{VB}}(\theta^*\s x) + \int_{K_n} \tilde{q}(\theta)
\left[\delta_n^{-1}(\theta-\theta^*)^\top V_{\theta^*}
\Delta_{n,\theta^*}+\frac{1}{2}(\delta_n^{-1}(\theta-\theta^*))^\top 
V_{\theta^*} \delta_n^{-1}(\theta-\theta^*)+o_P(1)\right]\dif \theta\\
\leq& p^{\mathrm{VB}}(\theta^*\s x) + C_6  + o_P(1)
\end{align*}
where $K_n$ is a compact set and $C_6 > 0$ is some constant.

For the set outside of this compact set $K_n$, choose
$\tilde{q}(\theta) =
\cN(\theta;\theta^*+\Delta_{n,\theta^*},
\delta_nV_{\theta^*}\delta_n).$ 
\begin{align}
&\int_{\mathbb{R}^d\backslash K_n} \tilde{q}(\theta)(\log p(\theta) + 
p^{\mathrm{VB}}(\theta\s x))\dif \theta\\
\leq &C_7 \int_{\mathbb{R}^d\backslash K_n} \cN(\theta;\theta^*+
\Delta_{n,\theta^*}, \delta_nV_{\theta^*}\delta_n)(\log p(\theta) + 
p^{\mathrm{VB}}(\theta\s x))\dif \theta\\
\leq &C_8 [\det(\delta_n)^{-1}\log(\det(\delta_n)^{-1})] 
\int_{\mathbb{R}^d\backslash K_n} \cN(\tilde{\theta}; \Delta_{n,\theta^*}, 
V_{\theta^*})
\log \pi^*(\tilde{\theta}\mid x) \det(\delta_n)\dif \tilde{\theta}\\
\leq &C_9 \log(\det(\delta_n)^{-1})]\int_{\mathbb{R}^d\backslash K_n}
[\pi^*(\tilde{\theta}\mid x)+o_P(1)]
\log \pi^*(\tilde{\theta}\mid x), V_{\theta^*})\dif \tilde{\theta}\\
\leq & C_{10} \log(\det(\delta_n)^{-1})]
\int_{\mathbb{R}^d\backslash K_n}
[\cN(\tilde{\theta}; \Delta_{n,\theta^*}, V_{\theta^*})+o_P(1)]
\log \cN(\tilde{\theta}; \Delta_{n,\theta^*}, 
V_{\theta^*})\dif \tilde{\theta}\\
\leq &o_P(1)
\end{align}
for some $C_7, C_8, C_9, C_{10} > 0$. The first two inequalities are
due to $\tilde{q}(\theta)$ centering at $\theta^*$ and a change of
variable step. The third and fourth inequality is due to
\Cref{lemma:vbvm} and Theorem 2 in \citet{piera2009convergence}. The
fifth inequality is due to a choice of fast enough increasing sequence
of compact sets $K_n$.

We repeat this argument for the lower bound of $\int
\tilde{q}(\theta)(\log p(\theta) + p^{\mathrm{VB}}(\theta\s x))\dif
\theta$. Hence the first claim is proved.

To prove the second claim \Cref{claim:compact}, we note that, for each
$\epsilon>0$, there exists an $N$ such that for all $n> N$ we have
$\int_{{||\theta -\mu||>M}}q(\theta)\dif \theta <\epsilon$ because
$\mathcal{Q}^d$ has a shrinking-to-zero scale. It leads to 
\[\int_{\mathbb{R}^d\backslash K}q^{\ddagger}(\theta)\dif \theta\leq
\int_{\mathbb{R}^d\backslash B(\mu,
M)}q^{\ddagger}(\theta)\dif \theta\leq \epsilon.\]

\section{Proof of \Cref{lemma:ivbnormal}}

\label{sec:ivbnormalproof}

\begin{proof}

To show the convergence of optimizers from two minimization problems,
we invoke $\Gamma$-convergence: if two functionals $\Gamma-$converge,
then their minimizer also converge. We refer the readers to Appendix C
of \citet{wang2018frequentist} for a review of $\Gamma$-convergence.

For notation convenience, we index the variational family by some
finite dimensional parameter $m$. The goal is to show
\[F_n(m):=\gls{KL}(q(\theta; m)||\pi^*(\theta\mid x))\]
$\Gamma$-converges to
\[F_0(m):=\gls{KL}(q(\theta; m)||\cN(\theta;\theta^*+
\delta_n\Delta_{n,\theta^*},\delta_nV_{\theta^*}^{-1}\delta_n)) 
- \Delta_{n,\theta^*}^\top V_{\theta^*}\Delta_{n,\theta^*}\] in
  $P_0$-probability as $n\rightarrow 0$. 

Write the densities in the mean field variational family in the
following form: $q(\theta) =
\prod^d_{i=1}\delta_{n,ii}^{-1}q_{h,i}(h),$ where $h =
\delta_n^{-1}(\theta - \mu)$ for some $\mu\in\Theta$. This form of
density is consistent with the change of variable step in
\Cref{sec:ivbconsistproof}.

\begin{assumption} We assume the following conditions on $q_{h,i}$:
\label{assumption:scoreint}
\label{assumption:vbpost_tech}
\begin{enumerate}
\item $q_{h,i}, i = 1, ..., d$ have continuous densities, positive and finite entropies, and $\int q_{h,i}'(h)\dif h < \infty, i = 1, ..., d.$\\
\item If $q_h$ is has zero mean, we assume $\int h^2\cdot
q_h(h)\dif h<\infty$ and $\sup_{z,x} |(\log p(z,x\mid\theta))''
|\leq C_{11}\cdot q_h(\theta)^{-C_{12}} $ for some $C_{11}, C_{12}
> 0;$ $|p^{\mathrm{VB}}(\theta\s x)''| \leq C_{13}\cdot q_h(\theta)^{-C_{14}}$
for some $C_{13}, C_{14} > 0.$

\item If $q_h$ has nonzero mean, we assume $\int h\cdot q_h(h)\dif
h<\infty$ and $\sup_{z,x} |(\log p(z,x\mid\theta))' |\leq
C_{11}\cdot q_h(\theta)^{-C_{12}} $ for some $C_{11}, C_{12} > 0$;
$|p^{\mathrm{VB}}(\theta\s x)' \leq C_{13}|\cdot q_h(\theta)^{-C_{14}}$ for
some $C_{13}, C_{14} > 0.$
\end{enumerate}
\end{assumption}

\Cref{assumption:vbpost_tech}.1 ensures that convergence in the
finite-dimensional parameter implied convergence in \gls{TV} distance
due to Eqs 64-68 of \citet{wang2018frequentist}.
\Cref{assumption:vbpost_tech} is analogous to Assumptions 2 and 3 of
\citet{wang2018frequentist}.

Leveraging the fundamental theorem of $\Gamma-$convergence
\citep{dal2012introduction,braides2006handbook}, the
$\Gamma$-convergence of the two functionals implies $m_n
\stackrel{P_0}{\rightarrow} m_0$;  $m_n$ minimizes $F_n$ and $m_0$
minimizes $F_0$. Importantly, this is true because
$\Delta_{n,\theta^*}^\top V_{\theta^*}\Delta_{n,\theta^*}$ is a
constant bounded in $P_0$ probability and does not depend on $m$. The
convergence in total variation then follows from
\Cref{assumption:scoreint}.

Therefore, what remains is to prove the $\Gamma$-convergence of the
two functionals.

We first rewrite $F_n(m,\mu)$.
\begin{align}
&F_n(m,\mu):=\gls{KL}(q(\theta;m,\mu)||\pi^*(\theta\mid x))\\
=& \log|\det(\delta_n)|^{-1} +\sum^d_{i=1}\mathbb{H}(q_{h,i}(h;m)) - 
\int q(\theta;m,\mu)\log p(\theta)\dif \theta-\int q(\theta;m,\mu)
p^{\mathrm{VB}}(\theta\s x)\dif \theta\nonumber\\
& + \log \int p(\theta)\exp(p^{\mathrm{VB}}(\theta\s x))\dif \theta\\
=&\log|\det(\delta_n)|^{-1} +\sum^d_{i=1}\mathbb{H}(q_{h,i}(h;m)) - 
\int q(\theta;m,\mu)\log p(\theta)\dif \theta-\int q(\theta;m,\mu)
p^{\mathrm{VB}}(\theta\s x)\dif \theta\nonumber\\
& + \left[\frac{d}{2}\log(2\pi) - \frac{1}{2}\log \det V_{\theta^*}+
\log\det(\delta_n)+p^{\mathrm{VB}}(\theta^*\s x)+\log p(\theta^*)+o_P(1)\right]\\
=&\sum^d_{i=1}\mathbb{H}(q_{h,i}(h;m)) -\int q(\theta;m,\mu)
p^{\mathrm{VB}}(\theta\s x)\dif \theta +\log p(\theta^*)
- \log p(\mu)\nonumber\\
&+\left[\frac{d}{2}\log(2\pi) - 
\frac{1}{2}\log \det V_{\theta^*}+p^{\mathrm{VB}}(\theta^*\s x) + o_P(1)\right]\\
=& \sum^d_{i=1}\mathbb{H}(q_{h,i}(h;m))- \int \delta_n^{-1}
(\theta-\theta^*)^\top V_{\theta^*}\Delta_{n,\theta^*}\cdot 
q(\theta;m,\mu)\dif \theta\nonumber \\
& + \int \frac{1}{2}(\delta_n^{-1}(\theta-\theta^*))^\top 
V_{\theta^*} \delta_n^{-1}(\theta-\theta^*)\cdot q(\theta;m,\mu)
\dif\theta - \left[\frac{d}{2}\log(2\pi) - \frac{1}{2}\log \det 
V_{\theta^*}+o_P(1)\right],
\label{eq:fn1}
\end{align}
due to algebraic operations and the \gls{LAN} condition of
$p^{\mathrm{VB}}(\theta\s x)$. To extend from the compact set $K$ to
the whole space $\mathbb{R}^d$, we employ the same argument as in
\Cref{eq:relaxcompact}.

Next we rewrite $F_0(m,\mu)$.
\begin{align*}
&\gls{KL}(q(\theta;m,\mu)||\cN(\theta;\theta^*+\delta_n
\Delta_{n,\theta^*},
\delta_nV_{\theta^*}^{-1}\delta_n))\\
=&\log|\det(\delta_n)|^{-1} +\sum^d_{i=1}\mathbb{H}(q_{h,i}(h;m)) 
+\int q(\theta;m,\mu) \log \cN(\theta;\theta^*+\delta_n
\Delta_{n,\theta^*},
\delta_nV_{\theta^*}^{-1}\delta_n)\dif \theta\\
=&\log|\det(\delta_n)|^{-1} +\sum^d_{i=1}\mathbb{H}(q_{h,i}(h;m)) 
+\frac{d}{2}\log(2\pi) - \frac{1}{2}\log \det V_{\theta^*}
+\log\det(\delta_n)\nonumber\\
&+ \int q(\theta;m,\mu)\cdot (\theta - \theta^*
-\delta_n\Delta_{n,\theta^*})^\top \delta_n^{-1}
V_{\theta^*}\delta_n^{-1}(\theta - 
\theta^*-\delta_n\Delta_{n,\theta^*})\dif \theta\\
=&\sum^d_{i=1}\mathbb{H}(q_{h,i}(h;m)) +\frac{d}{2}\log(2\pi) 
- \frac{1}{2}\log \det V_{\theta^*} + 
\Delta_{n,\theta^*}^\top V_{\theta^*}\Delta_{n,\theta^*} \nonumber\\
&- \int \delta_n^{-1}(\theta-\theta^*)^\top V_{\theta^*}
\Delta_{n,\theta^*}\cdot q(\theta;m,\mu)\dif \theta + 
\int \frac{1}{2}(\delta_n^{-1}(\theta-\theta^*))^\top V_{\theta^*} 
\delta_n^{-1}(\theta-\theta^*)\cdot q(\theta;m,\mu)\dif\theta.
\end{align*}

These representations of $F_0(m,\mu)$ and $F_n(m,\mu)$ leads to
\begin{align}
&F_0(m,\mu) - \Delta_{n,\theta^*}^\top V_{\theta^*}\Delta_{n,\theta^*}\\
= &\sum^d_{i=1}\mathbb{H}(q_{h,i}(h;m)) -\frac{d}{2}\log(2\pi) 
+ \frac{1}{2}\log \det V_{\theta^*} \nonumber\\
&- \int \delta_n^{-1}(\theta-\theta^*)^\top V_{\theta^*}
\Delta_{n,\theta^*}\cdot q(\theta;m,\mu)\dif \theta\nonumber\\
& + \int \frac{1}{2}(\delta_n^{-1}(\theta-\theta^*))^\top 
V_{\theta^*} \delta_n^{-1}(\theta-\theta^*)\cdot q(\theta;m,\mu)
\dif\theta\\
=&+\infty\cdot(1-\mathbb{I}_{\mu}(\theta^*)) + 
[\sum^d_{i=1}\mathbb{H}(q_{h,i}(h;m)) -\frac{d}{2}\log(2\pi) 
+ \frac{1}{2}\log \det V_{\theta^*}\nonumber\\
&- \int \delta_n^{-1}(\theta-\theta^*)^\top V_{\theta^*}
\Delta_{n,\theta^*}\cdot q(\theta;m,\mu)\dif \theta \nonumber\\
&+ \int \frac{1}{2}(\delta_n^{-1}(\theta-\theta^*))^\top 
V_{\theta^*} \delta_n^{-1}(\theta-\theta^*)\cdot q(\theta;m,\mu)
\dif\theta]\cdot 
\mathbb{I}_{\mu}(\theta^*)
\label{eq:f01}
\end{align}
where the last equality is due to
\Cref{assumption:scoreint}.

Comparing \Cref{eq:fn1} and \Cref{eq:f01}, we can prove the $\Gamma$
convergence. 

Let $m_n \rightarrow m$. When $\mu\ne\theta^*$, $\liminf_{n\rightarrow
\infty}F_n(m_n,\mu) = +\infty$. When $\mu = \theta^*$, we have
$F_n(m,\mu) = F_0(m,\mu)-\Delta_{n,\theta^*}^\top
V_{\theta^*}\Delta_{n,\theta^*}+o_P(1)$, which implies $F_0(m,\mu)
\leq \lim_{n\rightarrow \infty}F_n(m_n,\mu)$ in $P_0$
probability by \Cref{assumption:scoreint}. These implies the limsup
inequality required by $\Gamma$-convergence.

We then show the existence of a recovery sequence. When
$\mu\ne\theta^*$, $F_0(m,\mu) = +\infty$. When $\mu = \theta^*$, we
can simply choose $m_n = \theta^*$. Then we have $F_0(m,\mu) \leq
\lim_{n\rightarrow
\infty}F_n(p^{\mathrm{VB}},\mu)$ in $P_0$ probability and the
continuity of $F_n$. The $\Gamma$-convergence of the $F$ functionals
then follows from the limsup inequalities above and the existence of
recovery sequence.

Finally, we have \[\argmin F_0 = \argmin F_0 -
\Delta_{n,\theta^*}^\top V_{\theta^*}\Delta_{n,\theta^*}\] because
$\Delta_{n,\theta^*}^\top V_{\theta^*}\Delta_{n,\theta^*}$ does not
depend on $m$ or $\mu$. The convergence of \gls{KL} minimizers are
proved.

\end{proof}

\section{Details of the Simulations}

\label{sec:detailsim}

We follow the protocol as implemented in Stan. For HMC, we run four
parallel chains and use 10,000 burn-in samples, and determine mixing
using the R-hat convergence diagnostic (R-hat<1.01). For variational
Bayes, we run optimization until convergence (i.e. a local optimum).
We cannot confirm if the local optimal we reached is global. Further,
we conduct multiple parallel runs under each simulation setup and
report the mean and the standard deviation of ``RMSE'' or ``Mean KL.''
The error bars in \Cref{fig:varygamma} are the standard deviation
across different runs of the same simulation.

%% file: sec_app.tex
% !TEX root = vbmisspec.tex
\section{Applications of \Cref{thm:parammain,corollary:paramposteriorpred,thm:genmain,corollary:genposteriorpred}}
\label{sec:app}

We illustrate
\Cref{thm:parammain,corollary:paramposteriorpred,thm:genmain,corollary:genposteriorpred}
with three types of model misspecification: under-dispersion in
generalized linear model, component misspecification in Bayesian
mixture models, and latent dimensionality misspecification in Bayesian
stochastic block models. In all three cases, we show that the \gls{VB}
posterior converges to a point mass at the value that minimizes the
\gls{KL} divergence to the true data generating distribution; their
\gls{VB} posterior predictive distributions also converge to the true
posterior predictive distributions.

\subsection{Under-dispersion in Bayesian count regression}

The first type of model misspecification we consider is
under-dispersion. Suppose the data is generated by a Negative
Binomial regression model but we fit Poisson regression. We
characterize the asymptotic properties of the \gls{VB} posteriors of
the coefficient.

\begin{corollary}
Consider data $\mathcal{D} = \{(X_i, Y_i)\}_{i=1}^n$.
Assume the data generating measure $P_0$ has density 
\[p_0(Y_i\g X_i)=
\mathrm{Negative Binomial}\left(r, \frac{\exp(X_i^\top \beta_0)}{1+\exp(X_i^\top \beta_0)}\right)\]
for some constant $r$. Let $q^*(\beta)$ and $\hat{\beta}$ denote the
\gls{VB} posterior and its mean. We fit a Poisson regression model to
the data $\mathcal{D}$
\[p(Y_i\g X_i, \beta)= \mathrm{Poisson}(\exp(X_i^\top \beta)),\] with a prior that
satisfies \Cref{assumption:paramprior}. Let $\beta^*$ be the value
 such that \[\sum_{i=1}^n\E{P_0}{Y_i-\exp(X_i^\top\beta^*)X_i\g X_i} =
 0.\] Then we have
\begin{align*}
    \left\|q^*(\beta) - \cN\left(\beta\s \beta^*,
    \frac{1}{n}V^{'-1}\right)\right\|_{\mathrm{TV}}\stackrel{P_0}{
    \rightarrow} 0,
\end{align*}
and
\begin{align*}
\sqrt{n}(\hat{\beta} - \beta^*)\stackrel{d}{\rightarrow} \cN(0, V^{'-1}),
\end{align*}
where $V^{'} =
\mathrm{diag}\left[(X\exp(\frac{1}{2}X^\top\beta^*))^\top
(X\exp(\frac{1}{2}X^\top\beta^*))\right].$ Moreover, the model
misspecification error dominates the \gls{VB} approximation error in
the \gls{VB} posterior predictive distribution.

\label{corollary:countregression}
\end{corollary}

\begin{proof}
First, we notice that
\[\beta^* = \argmin_\beta \mathrm{KL}(p_0(Y\g X)\,||\,p(Y\g X, \beta)).\]
We then verify \Cref{assumption:paramlan,assumption:paramtest}.
\Cref{assumption:paramtest} is satisfied because the maximum
likelihood estimator converges to $\beta^*$
\citep{mccullagh1984generalized}. \Cref{assumption:paramlan} is
satisfied because the log likelihood of Poisson likelihood is
differentiable. Moreover, the log likelihood ratio is bounded a
squared integrable function and have a second order Taylor expansion
under $P_0$. Hence Lemma 2.1 of \citet{kleijn2012bernstein} implies
\Cref{assumption:paramlan}. Given
\Cref{assumption:paramprior,assumption:paramtest,assumption:paramlan},
we apply \Cref{thm:parammain,corollary:paramposteriorpred} and
conclude \Cref{corollary:countregression}.
\end{proof}

\subsection{Component misspecification in Bayesian mixture models}

The second type of misspecification we consider is component
misspecification. Suppose the data is generated by a Bayesian mixture
model where each component is Gaussian and shares the same variance.
But we fit a Bayesian mixture model where each component is a
$t$-distribution. We characterize the asymptotic properties of the
\gls{VB} posteriors.

\begin{corollary}
Consider data $\mathcal{D} = \{(X_i)\}_{i=1}^n$.
Assume the data generating measure $P_0$ has density 
\[p_0(x_i) = \sum_{c_i\in\{1,\ldots, K\}}\cN(x_i\g \mu_{c_i}, \Sigma)\cdot 
\mathrm{Categorical}(c_i\s 1/K),\] 
where $\Sigma$ and $K$ are two constants. Consider a mixture
model where each component is $t$-distribution with $\nu$ degree of
freedom, 
\[p(x_i\g \theta) =
\sum_{c_i\in\{1,\ldots, K\}}t(x_i\g \theta_{c_i}, \nu)\cdot
\mathrm{Categorical}(c_i\s 1/K),\]
with priors that satisfy \Cref{assumption:paramprior}. The goal is to
estimate the component centers $\theta\triangleq (\theta_1, \ldots,
\theta_K)$. Let $q^*(\theta)$ and $\hat{\theta}$ denote the
\gls{VB} posterior and its mean. Further denote the variational log
likelihood as follows:
\[m(\theta\s x) = 
    \sup_{q(c)\in\mathcal{Q}^n}\int q(c)
    \log\frac{p(x,c\mid\theta)}{q(c)}\dif c.\]
Under regularity conditions (A1-A5)
and (B1,2,4) of \citet{westling2015establishing}, we have
\begin{align*}
    \left\|q^*(\theta) - \cN\left(\theta^* + \frac{Y}{\sqrt{n}},
    \frac{1}{n}V_0(\theta^*)\right)\right\|_{\mathrm{TV}}\stackrel{P_0}{
    \rightarrow} 0,
\end{align*}
and
\begin{align*}
    \sqrt{n}(\hat{\theta} - \theta^*)\stackrel{d}{\rightarrow} Y,
\end{align*}
where $\theta^*$ satisfies 
\[\mathbb{E}_{P_0}{\nabla^2_\theta m(\theta^*\s x)}=0.\] 
The limiting distribution $Y$ is $Y\sim\cN(0, V(\theta^*)),$ where
$V(\theta^*) = A(\theta^*)^{-1}B(\theta^*)A(\theta^*)^{-1},$
$A(\theta) =
\mathbb{E}_{P_0}[\nabla^2_\theta m(\theta\s x)],$ and $B(\theta) =
\mathbb{E}_{P_0}[\nabla_\theta m(\theta\s x)\nabla_\theta
m(\theta\s x)^\top].$ The diagonal matrix $V_0(\theta^*)$ satisfies
$(V_{0}(\theta^*)^{-1})_{ii} = (A(\theta^*))_{ii}$. 
Moreover, the model misspecification error dominates the \gls{VB}
approximation error in the \gls{VB} posterior predictive distribution.

The specification of a mixture model is invariant to permutation among
$K$ components; this corollary is true up to permutations among the
$K$ components.
\label{corollary:appgmm}
\end{corollary}

\begin{proof}
First, we notice that $\theta^* = \argmin_\theta
\mathrm{KL}(p_0(X)\,||\,p(X\g \theta)).$ We then verify
\Cref{assumption:genlan,assumption:gentest,assumption:genmodellan}.
\Cref{assumption:gentest} is satisfied because the variational log
likelihood $m(x\s \theta)$ yields consistent estimates of $\theta$
\citep{westling2015establishing}.
\Cref{assumption:genlan,assumption:genmodellan} is satisfied by a
standard Taylor expansion of $m(x\s \theta)$ and $p(x\g \theta)$ at
$\theta^*$ (Eq. 128 of \citet{wang2018frequentist}). Given
\Cref{assumption:paramprior,assumption:gentest,assumption:genlan}, we
apply \Cref{thm:genmain,corollary:genposteriorpred} and conclude
\Cref{corollary:appgmm}.
\end{proof}

\subsection{Latent dimensionality misspecification in Bayesian
Stochastic Block Models}

The third type of misspecification we consider is latent
dimensionality misspecification. Suppose the data is generated by a
Bayesian stochastic block model with $K$ communities. But we fit a
Bayesian stochastic block model with only $K-1$ communities. We
characterize the asymptotic properties of the \gls{VB} posteriors.

\begin{corollary}
Consider the adjacency matrix $\mathcal{D} = \{(X_{ij})\}_{i,j=1}^n$
of a network with $n$ nodes. Suppose it is generated from the
stochastic block model with $K$ communities and parameters $\nu$ and
$\omega$. The parameter $\nu$ represents the odds ratio for a node to
belong to each of the $K$ communities. For simplicity, we assume the
odds ratio is the same for all the $K$ communities. The parameter
$\omega$ represents the $K\times K$ matrix of odds ratios; the
$(i,j)$-entry of $\omega$ is the odds ratio of two nodes being
connected if they belong to community $i$ and $j$ respectively.

We fit a stochastic block model with $K-1$ communities, whose prior
satisfies \Cref{assumption:paramprior}. Let $q^*_{\nu}(\nu_{(k-1)}),
q^*_{\omega}(\omega_{(k-1)})$ denote the
\gls{VB} posterior of $\nu_{(k-1)}$ and $\omega_{(k-1)}$, where
$\nu_{(k-1)}$ and $\omega_{(k-1)}$ are the odds ratios vector and
matrix for the $(K-1)$-dimensional stochastic block model. Similarly,
let $\hat{\nu}, \hat{\omega}$ be the \gls{VB} posterior mean. Then
\begin{align*}
    \left\|q^*_{\nu}(\nu)q^*_{\omega}(\omega)
    -\cN\left((\nu, \omega);
    (\nu^*, \omega^*)+
    (\frac{\Sigma^{-1}_1Y_1}{\sqrt{n\lambda_0}}, 
    \frac{\Sigma^{-1}_2Y_2}{\sqrt{n}}), V_n(\nu^*, \omega^*)
    \right)\right\|_{\mathrm{TV}}
    \stackrel{P_0}{\rightarrow}0
\end{align*}
where $\omega^*(a) = \log
\frac{\pi_{(k-1)}(a)}{1-\sum_{b=1}^{K-2}\pi_{(k-1)}(b)}$, $\nu^*(a,b) =\log
\frac{H_{(k-1)}(a,b)}{1-H_{(k-1)}(a,b)}, a, b = 1, \ldots, K-1$, where
$\pi_{(k-1)}(a)$ and $H_{(k-1)}(a,b)$ are the community weights vector
and the connectivity matrix where two smallest communities are merged.
The constant $\lambda_0$ is $\lambda_0 =
\mathbb{E}_{P_0}$(\text{degree of each node}), $(\log
n)^{-1}\lambda_0\rightarrow \infty.$ $Y_1$ and $Y_2$ are two zero mean
random vectors with covariance matrices $\Sigma_1$ and $\Sigma_2$,
where $\Sigma_1, \Sigma_2$ are known functions of $\nu^*, \omega^*$.
The diagonal matrix $V(\nu^*,
\omega^*)$ satisfies $V^{-1}(\nu^*, \omega^*)_{ii} =
\text{diag}(\Sigma_1, \Sigma_2)_{ii}$. Also,
\begin{align*}
    (\sqrt{n\lambda_0}(\hat{\nu} - \nu^*), 
    \sqrt{n}(\hat{\omega} - \omega^*))
    \stackrel{d}{\rightarrow}
    (\Sigma^{-1}_1Y_1, \Sigma^{-1}_2Y_2),
\end{align*}
Moreover, the model
misspecification error dominates the \gls{VB} approximation error in
the \gls{VB} posterior predictive distribution.

The specification of classes in \gls{SBM} is permutation invariant.
So the convergence above is true up to permutation with the $K-1$
classes. We follow \citet{bickel2013asymptotic} to consider the
quotient space of $(\nu, \omega)$ over permutations.
\label{corollary:sbm}
\end{corollary}

\begin{proof}
First, we notice that $(\nu^*, \omega^*) = \argmin_{\nu, \omega}
\mathrm{KL}(p_0(X_{ij})\,||\,p(X_{ij}\g \nu, \omega)).$ We then verify
\Cref{assumption:genlan,assumption:gentest,assumption:genmodellan}.
\Cref{assumption:gentest} is satisfied because the variational log
likelihood of stochastic block models yields consistent estimates of
$\nu, \omega$ even under under-fitted model
\citep{bickel2013asymptotic,wang2017likelihood}.
\Cref{assumption:genlan,assumption:genmodellan} is established by
Lemmas 2,3, and Theorem 3 of \citep{bickel2013asymptotic}. Given
\Cref{assumption:paramprior,assumption:gentest,assumption:genlan}, we
apply \Cref{thm:genmain,corollary:genposteriorpred} and conclude
\Cref{corollary:sbm}.
\end{proof}